\documentclass[11pt]{article}
\usepackage{longtable}% for long tables
\usepackage{microtype}
\usepackage[load-configurations=version-1]{siunitx} % newer version
\usepackage{siunitx}

\usepackage[square]{natbib}
\usepackage{hyperref}
\hypersetup{colorlinks,
            linkcolor=blue,
            citecolor=blue,
            urlcolor=magenta,
            linktocpage,
            plainpages=false}

\usepackage[margin=1.35in]{geometry}

\usepackage{graphicx}

\usepackage[english]{babel}  
\usepackage{times}
\usepackage[T1]{fontenc}
\usepackage[utf8]{inputenc} 
\DeclareUnicodeCharacter{00A0}{~}
  
\usepackage{amssymb,amsmath,amscd,amsfonts,amsthm,bbm,mathrsfs,yhmath}
\usepackage{url}
\usepackage{graphicx}
\usepackage{wrapfig}
\usepackage{subfigure}

\usepackage{tikz}
\usetikzlibrary{shapes, decorations,arrows,calc,arrows.meta,fit,positioning}
\tikzset{
    -Latex,auto,node distance =1. cm and 1. cm,semithick,
    state/.style ={ellipse, draw, minimum width = 0.4 cm},
    point/.style = {circle, draw, inner sep=0.04cm,fill,node contents={}},
    directed/.style={Latex-Latex,dashed},
    el/.style = {inner sep=2pt, align=left, sloped}
}

\usepackage{comment}
\theoremstyle{definition}

\newtheorem*{theorem*}{Theorem}
\newtheorem{theorem}{Theorem}
\newtheorem{proposition}{Proposition}
\newtheorem{lemma}{Lemma}
\newtheorem{remark}{Remark}

\newtheorem{definition}{Definition}
\newtheorem{corollary}{Corollary}
\newtheorem*{corollary*}{Corollary}
\newtheorem{condition}{Condition}

\def\z{{\phi(Z)}}
\def\E{{\mathcal {E}}}
\def\Ex{{\mathbb{E}}}
\def\H{{\mathcal{H}}}
\def\X{{\mathcal{X}}}
\def\Y{{\mathcal{Y}}}
\def\Z{{\mathcal{Z}}}
\def\L{L}
\def\A{{\mathcal{A}}}

\def\sumni{\sum_{i=1}^{n_1}}

\def\sumnj{\sum_{j=1}^{n_2}}
\def\summ{\sum_{l=1}^m}

\def\T{\top}
\def\op{O_p}
\def\di{{d_1}}
\def\dii{{d_2}}
\def\para{{||}}
\def\res{{\bot}}
\def\S{\mathbb S}

\DeclareMathOperator*{\argmin}{arg\,min}

\newcommand{\R}{{{\mathbb R}}} 

\usepackage{cleveref}
\usepackage{todonotes}

\newcommand{\td}[1]{\textcolor{magenta}{#1}}

\begin{document}



% \twocolumn[

% \aistatstitle{Nonlinear Causal Discovery via Kernel Anchor Regression}

% \aistatsauthor{ Wenqi Shi \And Wenkai Xu }

% \aistatsaddress{Tsinghua University \\ Department of Statistics, University of Oxford \And Department of Statistics, University of Oxford } ]

\title{Nonlinear Causal Discovery via Kernel Anchor Regression}

\author{Wenqi Shi$^{1}$ \and Wenkai Xu$^2$}
\date{%
    $^1$Department of Industrial Engineering, Tsinghua University\\[1ex]%
    $^2$Department of Statistics,
    University of Oxford\\[2ex]%
}
\maketitle

\begin{abstract}
Learning causal relationships 
% between variables 
is a fundamental problem in science. Anchor regression has been developed to address this problem for a large class of causal graphical models, 
though the relationships between the variables are assumed to be linear. 
% allowing distribution shift from observational data.
In this work, we tackle the nonlinear setting by proposing kernel anchor regression (KAR).
% an nonparametric procedure. 
Beyond the natural formulation using a classic two-stage least square 
% (2SLS) 
estimator, we also study an improved variant that involves nonparametric regression in three separate stages.
% estimation procedure that can outperform the two-stage counterpart in certain scenarios. 
We provide 
% theoretical
convergence results
% the consistency and convergence rate 
for the
proposed 
KAR estimators
% We also 
and the identifiability conditions for
% the proposed 
KAR to learn the nonlinear structural equation models (SEM).
Experimental results demonstrate the superior performances of the proposed KAR estimators over existing baselines.
\end{abstract}

\section{ Introduction}\label{sec:intro}
Causal relationships are concerned with consequences of actions or decisions;
thus, understanding these relationships can be the key ingredient in many scientific studies. For instance, medical practitioners need to know whether a treatment is effective to the target disease in clinical trials; econometricians ask whether a particular purchasing behaviour drives a change in Consumer Price Index (CPI); epidemiologists want to understand whether a government intervention policy has a positive effect on the pandemic. 
While the goal of revealing causal effects remains the same, the focus in causal relationships can differ by applications. 
To describe different aspects of the causal notion and design statistical procedures for inferring causal effects, various frameworks have been developed including 
Rubin's 
potential outcome framework \citep{rubin2004direct,rubin2005causal}, 
counterfactual distributions \citep{chernozhukov2013inference} and 
% structured causal models (SCM) 
Pearl's causal graphical models \citep{pearl2000models,pearl2016causal}. A succinct yet comprehensive introduction can be found in \citet{peters2017elements}.

Causality has also been an evolving field in machine learning community and
% Recently, 
machine learning techniques have been considered to improve the statistical procedures for causal discovery. In particular, nonparmetric independence \citep{gretton2005measuring} and conditional independence  \citep{fukumizu2007kernel} measures have been exploited to infer causal graphical models \citep{colombo2012learning, mooij2009regression}, especially with additive noise \citep{hoyer2008nonlinear, peters2014causal}.   Independent Component Analysis (ICA) methods \citep{hyvarinen2013independent, hyvarinen2017nonlinear} have been employed to identify causal relationship in both linear \citep{hyvarinen2010estimation, shimizu2006linear, shimizu2011directlingam} and non-linear settings \citep{monti2020causal, khemakhem2021causal}. Score matching \citep{hyvarinen2005estimation} has also been considered \citep{rolland2022score} for non-linear causal discovery.
Moreover, kernel methods, that utilize rich representation of reproducing kernel Hilbert space (RKHS), have been applied to tackle nonparametric estimation \citep{muandet2021counterfactual,singh2019kernel} and regression \citep{singh2019kernel, zhu2022causal} problems with causal implications.
Deep neural networks have also been attempted for  learning treatment effect \citep{johansson2020generalization, kallus2020deepmatch, louizos2017causal} or useful causal representations \citep{besserve2019counterfactuals,scholkopf2021toward, xu2020learning,xu2021deep}.

% The approach for causal notion remains un-unified, while existing attempts has achieved interesting theoretical guarantees and interpretable results based on graphical models, i.e. the causal graph \citep{pearl2000causality}.

% A loose definition: $X$ causes $Y$ whenever a change in $X$ results in change in $Y$.
Recently, an elegant and statistically robust approach formulates causality as an invariant risk minimization (IRM), see for example \citep{buhlmann2018invariance, peters2016causal}. The causal structure is thought to be invariant across the environment and robust under intervention. The IRM learning procedure \citep{arjovsky2019invariant} on the observational data is then formulated as a regularised empirical risk minimization (ERM) to achieve both in-distribution
performance and out-of-distribution generalization.
In particular, anchor regression \citep{rothenhausler2018anchor} has been developed under the IRM framework to tackle a very general class of causal graphical models with the confounders being partly 
(but not fully)
observed. By choosing different regularisation parameter, anchor regression is able to unify the ordinary least square (OLS) regression, partialling out (PA) regression, and instrumental variable (IV) regression. While existing works
\citep{oberst2021regularizing,rothenhausler2018anchor} mostly considered linear cases, we explore the non-linear setting for anchor regression \citep{kook2022distributional}. Specifically, we consider the nonparametric estimation to tackle non-linear features via RKHS functions.

The paper is structured as follows. In   \Cref{sec:background}, we review useful concepts including instrumental variable (IV), anchor regression (AR), and reproducing kernel Hilbert space (RKHS).
% multiple and
% structural equation model (SEM). 
Then we develop 
two versions of kernel anchor regression (KAR) estimators in \Cref{sec:kar}. Theoretical analysis on the estimators and the causal interpretation with nonlinear SEM are provided in \Cref{sec:analysis}. Experimental results for synthetic data and real-world applications are shown in \Cref{sec:simulation} followed by  concluding discussion
and future directions 
in \Cref{sec:conclusion}. The code for the experiments is 
% attached in the supplementary material
available in at \url{https://github.com/Swq118/Kernel-Anchor-Regression}.

\section{Background}\label{sec:background}

Directed Acyclic Graph (DAG) is a power class of graphical model for characterising conditional dependency structures and has been widely used for probabilistic modelling 
such as hidden Markov models
% (HMM) 
\citep{rabiner1986introduction}, latent variable models \citep{bishop1998latent} and topic models \citep{blei2012probabilistic}.
By enforcing certain Markov and faithfulness assumptions \citep{peters2011identifiability}, as well as noise structures \citep{hoyer2008nonlinear}, DAG models the causal relationships \citep{glymour2019review,spirtes2013causal} 
% (an example shown in \Cref{fig:anchor}) 
and the learning procedures have been developed \citep{colombo2012learning,spirtes2000causation,zhang2018learning}.
% Moreover, high-dimensional DAG settings have also been studied \citep{colombo2012learning}.
% With specific assumption on how the noise 

\subsection*{From Instrumental Variable to Anchor Regression}

Instrumental variable (IV) has been developed to incorporate endogenous explanatory variables in econometrics \citep{bowden1990instrumental} and then applied for estimating causal effect 
\citep{angrist1996identification}.
Consider the linear regression problem $Y = X \beta + \epsilon$. OLS assumes 
independence between noise $\epsilon$ and explanatory $X$ (the exogenous variable) and $\beta$ is estimated via minimizing
\begin{equation}\label{eq:ols_obj}
\beta^{OLS} = \argmin_{\beta}\mathbb E_{train}[\|Y-X\beta\|^2].
\end{equation}
% While the OLS assumes 
% independence between noise $\epsilon$ and explanatory $X$ 
% % $X\perp \epsilon$
% (the exogenous variable), 
The IV setting assumes explicit dependency between $X$  and $\epsilon$ via instrumental variable $Z$, i.e. $X = Z\theta + \varepsilon$ where $Z \perp\varepsilon$.
The two-stage least squares (2SLS)  procedure, widely used in economics, tackles the linear IV estimation by first regressing $Z$ over $X$ to get conditional means $\bar{X}(z) := \mathbb E[X|Z=z]$ and secondly 
% linearly 
regressing outputs $Y$ on these conditional
means\footnote{For the second stage, by writing 
$Y = X \beta + \epsilon  = 
\underset{\mathbb E[{X}| Z]}{\underbrace{Z\theta }}\beta + (\varepsilon\beta + \epsilon)
% {{(Z\theta)}}\beta + (\varepsilon\beta + \epsilon) = {\mathbb E[{X}| Z]}\beta + (\varepsilon\beta + \epsilon) 
,$ then the regressor is independent of noise and the OLS estimator can then apply.}.
This corresponds to minimizing the projected least square objective, 
\begin{equation}\label{eq:iv_obj}
\beta^{IV} = \argmin_{\beta}\mathbb E_{train}[\|P_Z(Y-X\beta)\|^2].
\end{equation}
$P_Z$ denotes the projection to $Z$ where $P_{Z=z}(X) = \mathbb E [X|Z=z] = \bar X(z)$.
2SLS works well when the underlying assumptions hold. 
The corresponding DAG is shown in \Cref{fig:anchor} with only solid lines. 
In practice, the relation between
$Y$ and $X$ may not be linear, nor may be the relation between $X$ and $Z$. Nonlinear IV has been explored \citep{bennett2019deep, centorrino2019nonparametric,hartford2017deep, singh2019kernel, xu2020learning,zhu2022causal}.

% \begin{wrapfigure}{r}{0pt}
% \centering
% \begin{tikzpicture}
%     \node[state] (1) {$Z$};
%     \node[state] (2) [right =of 1] {$X$};
%     \node[state] (3) [right =of 2] {$Y$};
%     \node[state] (4) [above right =of 2,xshift=-0.9cm,yshift=-0.3cm] {$C$};

%     \path (1) edge node[above] {} (2);
%     \path (2) edge node[above] {} (3);
%   % \path[bidirected] (2) edge[bend left=60] node[above] {} (3);
%     \path (4) edge node[el,above] {} (2);
%     \path (4) edge node[el,above] {} (3);
%     \path (1) edge[dashed, bend right=30] node[above] {} (3);
%     \path (1) edge[dashed, bend left=10] node[above] {} (4);
% \end{tikzpicture}
% \caption{IV model (solid lines only) and anchor regression model (with dashed lines).
% }\label{fig:anchor}
% \end{wrapfigure}

\begin{figure}[t!]
    \centering
\begin{tikzpicture}
    \node[state] (1) {$Z$};
    \node[state] (2) [right =of 1] {$X$};
    \node[state] (3) [right =of 2] {$Y$};
    \node[state] (4) [above right =of 2,xshift=-0.9cm,yshift=-0.3cm] {$C$};

    \path (1) edge node[above] {} (2);
    \path (2) edge node[above] {} (3);
   % \path[bidirected] (2) edge[bend left=60] node[above] {} (3);
    \path (4) edge node[el,above] {} (2);
    \path (4) edge node[el,above] {} (3);
    \path (1) edge[dashed, bend right=30] node[above] {} (3);
    \path (1) edge[dashed, bend left=10] node[above] {} (4);
\end{tikzpicture}
\caption{IV regression (solid lines only) and anchor regression (with dashed lines).
}\label{fig:anchor}
% \vspace{-0.5cm}
\end{figure}
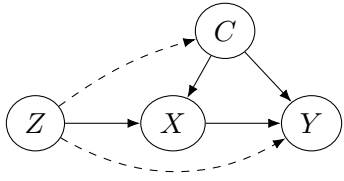

However, $Y$'s dependency on $Z$ may not be solely through $X$, i.e. as the dashed lines from $Z$ to $Y$ in \Cref{fig:anchor} indicate, $Y$ may depend on $Z$ directly, even though the strength of such dependency may remain unknown. The latent confounder $C$ may not be independent of 
%the measured confounder 
$Z$, as indicated by dashed line from $Z$ to $C$ in \Cref{fig:anchor}. Incorporating such dependency structures tackles a much more general class of DAG, e.g. IV is a special case. To estimate $\beta$, anchor regression has been proposed \citep{rothenhausler2018anchor} that effectively combines \Cref{eq:ols_obj} and \Cref{eq:iv_obj}. 
% Denote $\mathbb E_{train}$ by $\mathbb E$ when no ambiguity arise. 
For chosen regularization parameter $\gamma$ and identity operator $Id(Z):=Z$, 
\begin{align}
\beta^{\gamma} = \argmin_{\beta}\mathbb E_{train}[\|(Id-P_Z)(Y-X\beta)\|^2] \label{eq:ar_obj1}\\
% \mathbb E\|(Y-X\beta)\|^2 
+ \gamma \mathbb E_{train}[\|P_Z(Y-X\beta)\|^2].
\label{eq:ar_obj2}
\end{align}
Here, $\gamma\geq 0$ can be thought of the level of dependencies of $Y$ from $Z$ variable\footnote{The smaller $\gamma$ value dashed line, the stronger the dependency, i.e. the more solid dashed line from $Z$ to $Y$.}.
By setting different $\gamma$ values, anchor regression resembles classical settings, i.e. $\gamma=1$ corresponds to OLS, $\beta^1 = \beta^{OLS}$;
$\gamma\to \infty$ corresponds to IV, $\beta^{\to \infty}:= \lim_{\gamma \to \infty}\beta^{\gamma} = \beta^{IV}$;
$\gamma=0$ corresponds to the "partialling out" setting where only residuals between regression of $Z$ to $X$ and $Y$ are of interest.

\subsection*{Kernel-based 
% nonparametric 
Methods}

Kernel methods employ functions in RKHS to tackle various statistical and machine learning tasks with nonlinear features \citep{hofmann2008kernel}, e.g. kernel ridge regression, support vector machine\citep{scholkopf2018learning,steinwart2008support}, etc.
Functions in RKHS have also been developed to represent and characterize distributions, via kernel mean embedding \citep{muandet2017kernel}. For probability measure $p$, and kernel $k$ associated with RKHS $\H$, the mean embedding $\mu_p := \int k(x,\cdot) dp(x) \in \H$. This notion has been widely used to compare distributions, e.g. via maximum-mean-discrepancy (MMD) \citep{gretton2012kernel}.
With $p$ being a conditional distribution, conditional mean embedding \citep{song2009hilbert} has also been considered for learning and regression problems.
Various techniques have also been developed to formulate and learn operators to manipulate conditional mean embeddings \citep{fukumizu2007kernel,grunewalder2012conditional}.
With the rich representation of nonlinear features, RKHS functions are also applicable of learning distribution directly via distribution regression \citep{szabo2015two, szabo2016learning}.
% whose analysis and techniques are closely related to what we are using to analyse the KAR estimators.

\section{Kernel Anchor Regression}\label{sec:kar}
To capture the non-linear features in the DAG, 
% in the causation,
we kernelize the anchor regression framework by
utilizing the rich feature representation of RKHS functions.
% introducing the reproducing kernel Hilbert spaces (RKHS) model. 
The kernelizing procedure is inspired from kernel instrumental variable (KIV) \citep{singh2019kernel} where the operators are learned for conditional mean embedding in two separate regression stages. The DAG representation is illustrated in \Cref{fig:KAR}.

Let $k_\X: \X \times \X \rightarrow \R$, $k_\Z: \Z \times \Z \rightarrow \R$ be measurable positive definite kernels corresponding to 
% scalar-valued 
RKHS $\H_\X$ and $\H_\Z$. Denote the feature maps
$
\psi: \X \rightarrow \H_\X, x \rightarrow k_\X(x, \cdot)$ and 
$\phi: \Z \rightarrow \H_\Z, z \rightarrow k_\Z(z, \cdot).
$
%As shown in Figure~\ref{fig:KAR}, our model is general, as we allow unobserved confounding factor $C$, and anchors $Z$ to influence treatment $X$, outcome $Y$ and confounding factor $C$. 
Let $P_{\z}$ and $Id$ denote the $L_2$-projection on the linear span from the components of $\phi(Z)$ and the identity operator, respectively. Denote $H: \H_\X \to \Y$ as the conditional operator we aim to learn.
Then for $\gamma \geq 0$, define the population-level kernel anchor regression operator $H^{\gamma}$ as
\begin{align}
    {H}^{\gamma}=\argmin_{H}  \Ex[\|(Id-P_{\z})(Y-H\psi(X))\|^2] \nonumber\\
     + \gamma \Ex[\|P_{\z}(Y-H\psi(X))\|^2].
     \label{eq:KAR}
\end{align}
% and its empirical analogue $\widehat{H}^{\gamma}$ 
% \begin{align}
%     \widehat{H}^{\gamma}=\argmin_{H} (  \Ex_{\widehat{p}}[((P_{\z}-Id)(Y-H\psi(X)))^2]\\  + \gamma \Ex_{\widehat{p}}[(P_{\z}(Y-H\psi(X))^2]  ).
% \end{align}
To unravel $P_\z$, both IV and AR estimators applied the two-stage procedure, where the first stage is to estimate the projection operator $P_\z$ and the second stage is to perform the projection adjusted regression.

\begin{figure}[!tb]
  \centering
\centering
\begin{tikzpicture}
    \node[state] (1) {$Z$};
    \node[state] (2) [right =of 1]{$\mathcal{H}_\Z$};
    \node[state] (3) [below =of 1] {$X$};
    \node[state] (4) [right =of 3] {$\mathcal{H}_\X$};
    \node[state] (5) [right =of 4] {$C$}; %Hidden confounder
    \node[state] (6) [below =of 4] {$Y$};
    
    \path (1) edge node[above] {$\phi$} (2);
    \path (2) edge node[right] {} (4);
    \path (3) edge node[above] {$\psi$} (4);
    \path (2) edge[dashed] node[above] {} (5);
    \path (5) edge node[above, xshift=0.9cm] {} (4);
    %\path (1) edge node[above] {} (3);
    %\path (3) edge node[right, xshift=0.2cm] {$h$} (6);
    \path (4) edge node[right] {} (6);
    \path (5) edge node[above] {} (6);
    \path (2) edge[dashed, bend left=45] node[above] {} (6);
\end{tikzpicture}
\caption{
% \small 
DAG representation for kernel anchor regression. 
% Solid lines represent the dependencies that are also present in the instrumental variable model, while the dashed arrows represent the additional dependencies taken into account in the kernel anchor regression model.
}
\label{fig:KAR}
% \vspace{-0.4cm}
\end{figure}
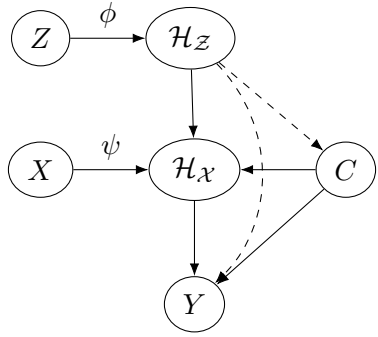

\subsection{Projection Stage}\label{sec:projection}
% \subsubsection{Stage I and Stage II}

The projection stage aims to tackle $P_\z$ by
% In the first two stages, we 
transforming the problem of learning 
$P_\z \psi(X)$ and $P_\z Y$ into two separate
% vector-valued 
kernel ridge regressions.
% , where the hypothesis spaces are the vector-valued RKHS $\H_\Gamma$ and $\H_\Theta$ of operators mapping $\H_\X$ and $\Y$ to $\H_\Z$, respectively. 
% We now state the loss function for optimizing operator $E$. The optimal 
Let operators $E_{X}:\H_\Z \to \H_\X$ and $E_{Y}:\H_\Z \to \Y$ be the projections to learn\footnote{We note that due to the explicit dependency from $Z$ to $Y$, $P_\z Y$ needs to be treated separately from $P_\z \psi(X)$. This is different from the IV setting where $P_\z Y = Y$ as the edge from $Z$ to $Y$ in the DAG is absent.}; 
$\alpha_1, \alpha_2 > 0$ be regularization parameters. The  
% regularized learning 
objectives regularized by Hilbert-Schmidt (HS) norm are
% The projection $P_\z$ solves
\begin{align}\label{eq:proj_x}
    % E_{X,\alpha_1}^{\ast} =  \argmin 
    &\E_{\alpha_1}(E_X) = \Ex
    % _{(Z, X)} 
    \Vert \psi(X) - E_X \phi(Z) \Vert_{\H_\X}^2+ \alpha_1 \Vert E_X \Vert_{HS}^2,\\
    % E_{Y,\alpha_2}^{\ast} = \argmin 
    &\E_{\alpha_2}(E_Y) = \Ex \Vert Y - E_Y \phi(Z) \Vert_{{\Y}}^2+ \alpha_2 \Vert E_Y \Vert_{HS}^2.\label{eq:proj_y}
\end{align}
% The optimal operators are
Denote 
the optimal operators for the population risks as
$E^p_{\alpha_1,X} =  \argmin
_{E_X} 
\E_{\alpha_1}(E_X)$, and
$E^p_{\alpha_2,Y} = \argmin
_{E_Y} 
\E_{\alpha_2}(E_Y)$.
% We then present the estimations via empirical samples of two variants.
We then consider two variants of empirical risks and their corresponding estimations. 

% \paragraph{Empirical estimations}
\subsubsection{Disjoint sample sets projection}\label{sec:disjoint_proj}
Firstly, we treat two ridge regression in \Cref{eq:proj_x} and  \Cref{eq:proj_y} independently, by using two \textit{disjoint}
sets of samples 
$\S_1 = \{(x_i, z_i)\}_{i \in [n_1]}$ and $\S_2 = \{(y_j, z_j)\}_{j \in [n_2]}$. The empirical forms for \Cref{eq:proj_x} and \Cref{eq:proj_y} are
\begin{align}\label{eq:proj_emp1}
% \tiny
% \E^{n_1}_{\alpha_1}(E_X) = 
\frac{1}{n_1}\sum_{i \in [n_1]}\Vert \psi(x_i) - E_X \phi(z_i) \Vert_{\H_\X}^2 
+ \alpha_1 \Vert E_X \Vert_{HS}^2, \\
\label{eq:proj_emp2}
% \E^{n_2}_{\alpha_2}(E_Y) =   
\frac{1}{n_2}\sum_{j \in [n_2]} \Vert y_j - E_Y \phi(z_j) \Vert_{\Y}^2 
+ \alpha_2 \Vert E_Y \Vert_{HS}^2&.
\end{align}
Denote by
% $\Phi_{1,Z} = (k_\Z(z_1,\cdot), \dots, k_\Z(z_{n_1},\cdot)), z_i \in \S_1$; $\Phi_{2,Z} = (k_\Z(z_1,\cdot), \dots, k_\Z(z_{n_2},\cdot)), z_j \in \S_{2}$; 
$\Phi_{1,Z} = (\phi(z_1), \dots, \phi(z_{n_1}))$, 
$\{z_i\}_{i\in[n_1]} \subset \S_1$; 
$\Phi_{2,Z} = (\phi(z_1), \dots, \phi(z_{n_2}))$, $\{z_j\}_{j\in[n_2]} \subset \S_{2}$; 
their 
corresponding 
gram matrices
$K_{1,ZZ} = \Phi_{1,Z}^\T\Phi_{1,Z} \in \R^{n_1 \times n_1}$ and $K_{2,ZZ} = \Phi_{2,Z}^\T\Phi_{2,Z}\in \R^{n_2 \times n_2}$. % from  
% $\{z_i\}\subset \S_1 $ and $\{z_j\}\subset \S_2$ respectively; 
% $\Psi_{1,X}=(k_\X(x_1,\cdot), \dots, k_\X(x_{n_1},\cdot)), x_i \in \S_1$ and 
Denote $\Psi_{1,X}=(\psi(x_1), \dots, \psi(x_{n_1}))$, $\{x_i\}_{i\in[n_1]} \subset \S_1$ and $Y_2=(y_1,\dots, y_{n_2})$, $\{y_j\}_{j\in[n_2]}\subset \S_2$. 
By the standard regression formula, the optimal operators to minimize \Cref{eq:proj_emp1} and \Cref{eq:proj_emp2} are
\begin{align}\label{eq:est_proj_x1}
    %\widehat 
    E_{\alpha_1, X}^{n_1} &= \Psi_{1,X}(K_{1,ZZ} + n_1\alpha_1 I)^{-1}\Phi_{1,Z}^\T, \\
    %\widehat 
    E_{\alpha_2, Y}^{n_2} &= Y_2 (K_{2,ZZ} + n_2 \alpha_2 I)^{-1}\Phi_{2,Z}^\T,
    \label{eq:est_proj_y1}
\end{align}
% {\wenqi I delete the hat here but not the corresponding terms below. Probably we should discuss whether we should delete all the hats.}\wk{agreed}
% where $K_{1,ZZ}$, $K_{2,ZZ}$ are the empirical kernel matrix for anchors $Z$ from stage I and stage II observations, respectively, $\Phi_{1,Z}$,  and $\Psi_{1,X}$ are feature matrices of $Z$ and $X$ from stage I observations, $\Phi_{2,Z}$ is the feature matrix of $Z$ from stage II observations, $y_2$ is the vectorization of $\{y_{2,i}\}$.
% Note that due to the disjoint \wk{set of}  i.i.d. samples of $Z$ are used, the gram matrix $K_{1,ZZ}$ and $K_{2,ZZ}$ are independent. 
% We can think of \Cref{eq:est_proj_x1} and \Cref{eq:est_proj_y1} have distinct projections onto feature spaces of $\phi(Z)$.
where the superscripts $n_1, n_2$ explicitly reveal sample sizes.
We note that the  projections $P_{\phi(Z)}$ are estimated differently
% are different 
for $P_{\phi(Z)}\psi(X)$ and $P_{\phi(Z)}Y$, through $(K_{1,ZZ} + n_1\alpha_1 I)^{-1}$ and $(K_{2,ZZ} + n_2 \alpha_2 I)^{-1}$, respectively.
$K_{1,ZZ}$ and $K_{2,ZZ}$ are independent due to the use of disjoint i.i.d. sample sets of $Z$.

\subsubsection{Joint sample set projection}\label{sec:joint_proj}
On the other hand, we can also consider the projection analogous to \citet{rothenhausler2018anchor} where we jointly consider the samples used for both projections, i.e. projecting onto the same $\phi(Z)$ subspace. 
Setting $n = n_1 + n_2$ 
% for a fair comparison, 
and $\alpha_1 = \alpha_2 = \alpha$, we consider the joint sample set $\S=\{(x_i, y_i, z_i)\}_{i \in [n]}$ and the empirical risks
\begin{eqnarray}\label{eq:proj}
    \frac{1}{n}\sum_{i \in [n]}\Vert \psi(x_i) - E_X \phi(z_i) \Vert_{\H_\X}^2+ \alpha \Vert E_X \Vert_{HS}^2, \\
    \frac{1}{n}\sum_{i \in [n]}\frac{1}{n}\Vert y_i - E_Y \phi(z_i) \Vert_{\Y}^2+ \alpha \Vert E_Y \Vert_{HS}^2.
\end{eqnarray}

Denote $K_{ZZ}  \in \R^{n \times n}$ as the gram matrix from  
$\{z_i\}_{i \in [n]}\subset \S $; 
$\Phi_{Z} = (\phi(z_1), \dots$, $\phi(z_n)), \{z_i\}_{i\in[n]} \subset \S$; $\Psi_{X}=(\psi(x_1), \dots, \psi(x_{n})), \{x_i\}_{i\in[n]} \subset \S$ and $Y=(y_1,\dots, y_{n}), y_i\in \S$. 
%$\Phi_{Z} = (k_\Z(z_1,\cdot), \dots, k_\Z(z_{n},\cdot)), z_i \in \S$; $\Psi_{X}=(k_\X(x_1,\cdot), \dots, k_\X(x_{n},\cdot)), x_i \in \S$ and $Y=(y_1,\dots, y_{n}), y_i\in \S$. 
Then we have
\begin{align}\label{eq:est_proj_x2}
    % \widehat 
    E_{\alpha, X}^{n} &= \Psi_{X}(K_{ZZ} + n\alpha I)^{-1}\Phi_{Z}^\T, \\
    % \widehat 
    E_{\alpha, Y}^{n} &= Y^\T (K_{ZZ} + n \alpha I)^{-1}\Phi_{Z}^\T.
    \label{eq:est_proj_y2}
\end{align}
By setting the same level of regularisation, we can see that the $P_{\phi(Z)}$ projection, through $(K_{ZZ} + n\alpha I)^{-1}$, are the same for $P_{\phi(Z)}\psi(X)$ and $P_{\phi(Z)}Y$.

% \iffalse
% minimize the expected discrepancy
% \begin{eqnarray*}
%     \E_{1,X} (E) &=& \Ex_{(Z, X)} \Vert \psi(X) - E^* \phi(Z) \Vert_{\H_\X}^2,\\
%     \E_{2,Y} (E) &=& \Ex_{(Z, Y)} \Vert Y - E^* \phi(Z) \Vert_{\Y}^2,\\
%     E_{X} &=&\argmin \E_{1,X} (E), \\
%     E_{Y} &=& \argmin \E_{2,Y} (E). 
% \end{eqnarray*}
% %where $\rho$ denotes the true distribution behind $(Z, X, Y)$. 
% Next we impose an $L_2$ regularization. For stage I, the regularized target operator and its empirical analogue are given by
% \begin{eqnarray*}
%     \E_{\alpha_1,X} (E) &=& \E_{1, X} (E) + \alpha_1 \Vert E \Vert_{L_2(\H_\X, \H_\Z)}^2,  \\
%     \E_{\alpha_1,X}^{n_1} (E) &=& \frac{1}{n_1} \sumni \Vert \psi(x_{1,i}) - E^* \phi(z_{1,i}) \Vert_{\H_\X}^2 \\
%     && + \alpha_1 \Vert E \Vert_{L_2(\H_\X, \H_\Z)}^2,\\
%     E_{\alpha_1, X} &=& \argmin \E_{\alpha_1,X}(E), \\
%     E_{\alpha_1, X}^{n_1} &=& \argmin \E_{\alpha_1,X}^{n_1}(E),
% \end{eqnarray*}
% where $\alpha_1 \geq 0$ is a given regularization parameter. For stage II,
% \begin{eqnarray*}
%     \E_{\alpha_2,Y} (E) &=& \E_{2, Y}(E) + \alpha_2 \Vert E \Vert_{L_2(\Y, \H_\Z)}^2,  \\
%     \E_{\alpha_2,Y}^{n_2} (E) &=& \frac{1}{n_2} \sumnii \Vert y_{2,i} - E^* \phi(z_{2,i}) \Vert_{\Y}^2 \\
%     && + \alpha_2 \Vert E \Vert_{L_2(\Y, \H_\Z)}^2,\\
%     E_{\alpha_2, Y}  &=& \argmin \E_{\alpha_2,Y}, \\
%     E_{\alpha_2, Y}^{n_2} &=& \argmin \E_{\alpha_2,Y}^{n_2},
% \end{eqnarray*}
% where $\alpha_2 \geq 0$ is also a given regularization parameter.
% \fi

\subsection{Regression Stage}\label{sec:regression}
% \subsubsection{Stage III}
With the learned projections $P_{\phi(Z)}\psi(X)$ and $P_{\phi(Z)}Y$, we can now tackle the overall objective in \Cref{eq:KAR}.

Denote $\E(E_X)$ and $\E(E_Y)$ as the unregularized version of \Cref{eq:proj_x} and \Cref{eq:proj_y}; $E^p_X$ and $E^p_Y$ their corresponding optimal operators, respectively. 
For given 
% parameter 
$\gamma$, define the transformed input and output as 
\begin{equation}\label{eq:transform_x}
 \psi_\gamma(X) = \psi(X) - E^p_{X} \phi(Z) + \sqrt{\gamma} E^p_{X} \phi(Z) \in \H_\X,    
\end{equation}
% and
\begin{equation}\label{eq:transform_y}
Y_\gamma = Y - E^p_{Y} \phi(Z) + \sqrt{\gamma} E^p_{Y} \phi(Z) \in \Y.
\end{equation} 
\begin{proposition}[Equivalence]
Let $H:\H_\X \to \Y$, and 
%consider the risk on the transformed variable in \Cref{eq:transform_x} and \Cref{eq:transform_y}
consider the regression of transformed output in \Cref{eq:transform_y} on transformed input in \Cref{eq:transform_x}
\begin{equation}\label{eq:transform_obj}
    \E^\gamma(H) = \Ex_{(Z,X,Y)} \Vert Y_\gamma - H \psi_\gamma(X) \Vert_\Y^2.
\end{equation}
The solution to \Cref{eq:transform_obj} is equivalent to the KAR estimator in \Cref{eq:KAR}, i.e.
% \begin{equation}
$H^\gamma = \argmin_{H} \E^\gamma(H).$
% \end{equation}
\end{proposition}
The proof is by expanding the projection operator $E^p_X$ and $E^p_Y$, which is similar to the linear case in \citet{rothenhausler2018anchor}.

% The problem of learning $H^\gamma$ can be transformed into a scalar-valued kernel ridge regression on transformed data set, where the hypothesis space is the vector-valued RKHS $\H_\Omega$ of operators mapping $\H_\X$ to $\Y$. Given parameter $\gamma$, define transformed input $X_\gamma = \psi(X) - E_{ X}^* \phi(Z) + \sqrt{\gamma} E_{ X}^* \phi(Z) $,
% and transformed outcome $Y_\gamma = Y - E_{ Y}^* \phi(Z) + \sqrt{\gamma} E_{ Y}^* \phi(Z)$. The unconstrained solution is
% \begin{eqnarray*}
%     \E^\gamma(H) &=& \Ex_{(Z,X,Y)} \Vert Y_\gamma - H X_\gamma \Vert_\Y^2,\\
%     H^\gamma &=& \argmin \E^\gamma(H).
% \end{eqnarray*}

With regularization parameter $\xi \geq 0$, \Cref{eq:transform_obj} has the kernel ridge regression form defined as
\begin{equation}\label{eq:transform_reg}
        \E_\xi^\gamma(H) = \Ex_{(Z,X,Y)} \Vert Y_\gamma - H \psi_\gamma(X) \Vert_\Y^2 + \xi \Vert H \Vert_{HS}^2.
\end{equation}
% \wk{I re-anotate the different notations here. please check the regularizer notation, which one you used in different stage of proof is more convenient.}
% Consider $E_{\alpha_1,X}^{\ast}$, $E_{\alpha_2, Y}^{\ast}$ the solutions for \Cref{eq:proj_x} and \Cref{eq:proj_y} repectively. 
\iffalse
The transformed input and output samples are
$$\psi_{\gamma,l}(x) = \psi(x_{l}) + (\sqrt{\gamma}-1)  E_{X}  \phi(z_{l}) \in \H_\X,$$ 
% and 
$$y_{\gamma,l} = y_{l} + (\sqrt{\gamma}-1) E_{Y}  \phi(z_{l}) \in \Y.$$ 
\fi
The regression stage is formulated regardless how the projections are estimated in \Cref{sec:projection}. For the empirical version, we consider the estimated operators 
% from the projection stage 
$\widehat E_X \in \{ E_{\alpha_1, X}^{n_1},  E_{\alpha, X}^{n}\}$ 
% in \Cref{eq:est_proj_x1} and \Cref{eq:est_proj_x2}; 
and $\widehat E_Y \in \{ E_{\alpha_2, Y}^{n_2}, E_{\alpha, Y}^{n}\}$. 
% in \Cref{eq:est_proj_y1} and \Cref{eq:est_proj_y2}. 
With samples $\S^m=\{(x_l, y_l, z_l)\}_{l \in [m]}$, 
we compute the transformed inputs and outputs as $$\widehat \psi_{\gamma,l}(x) = \psi(x_{l}) + (\sqrt{\gamma}-1) \widehat E_{X}  \phi(z_{l}) \in \H_\X,$$ 
% and 
$$\widehat y_{\gamma,l} = y_{l} + (\sqrt{\gamma}-1) \widehat E_{Y} \phi(z_{l}) \in \Y.$$ 
% respectively; 
The empirical risk has the form
\begin{equation}\label{eq:transform_reg}
        \widehat \E_\xi^{\gamma,m}(H) = \frac{1}{m}\sum_{l \in [m]}\Vert \widehat y_{\gamma,l} - H \widehat \psi_{\gamma,l}(x) \Vert_\Y^2 + \xi \Vert H \Vert_{HS}^2,
\end{equation}
% \begin{eqnarray*}
%     \E_\xi^\gamma(H) &=& \E^\gamma(H) + \xi \Vert H \Vert_{\H_\Omega}^2,\\
%     \E_\xi^{\gamma,m}(H) &=& \frac{1}{m} \summ  \Vert  y_{\gamma,3,i} - H x_{\gamma,3,i} \Vert_{\Y}^2 + \xi \Vert H \Vert_{\H_\Omega}^2,\\
%     H_\xi &=& \argmin \E^\gamma_\xi(H), \\
%     H_\xi^m &=& \argmin \E_\xi^{\gamma,m}(H).
% \end{eqnarray*}
% However, we do not directly observe the conditional expectation operator $E_{X}$ and $E_{Y}$, so we approximate it using the estimates from first two stages. Let
% \begin{eqnarray*}
%     \hat x_{\gamma,3,i} &=&  \psi(x_{3,i}) + (\sqrt{\gamma}-1) (E_{\alpha_1, X}^{n_1})^*  \phi(z_{3,i}),\\
%     \hat y_{\gamma,3,i} &=& y_{\gamma,3,i} + (\sqrt{\gamma}-1) (E_{\alpha_2, Y}^{n_2})^* \phi(z_{3,i}),\\
%     \hat \E_\xi^{\gamma,m}(H) &=& \frac{1}{m} \summ  \Vert  {\hat y}_{\gamma,3,i} - H \hat x_{\gamma,3,i} \Vert_{\Y}^2 + \xi \Vert H \Vert_{\H_\Omega}^2,\\
%     \hat H_\xi^{\gamma,m} &=& \argmin \hat \E_\xi^{\gamma,m}(H).
% \end{eqnarray*}
$$
\widehat H_\xi^{\gamma,m} = \argmin \widehat \E_\xi^{\gamma,m}(H).
$$
Denote $\widehat \Psi_{\gamma} = (\widehat \psi_{\gamma,1}(x),\dots,\widehat \psi_{\gamma,m}(x))$ and 
% the corresponding
its gram matrix $K_{\widehat \Psi_\gamma  \widehat \Psi_\gamma}=\widehat \Psi_\gamma^{\T}\widehat \Psi_\gamma \in \R^{m\times m}$; $\widehat Y_\gamma=(\widehat y_{\gamma,1},\dots,\widehat y_{\gamma,m})$. Again, by standard regression formula,
\begin{equation}\label{eq:KAR_estimator}
    \widehat H_\xi^{\gamma,m} = \widehat{Y}_\gamma (K_{\widehat \Psi_\gamma  \widehat \Psi_\gamma} + m\xi I)^{-1} \widehat \Psi_\gamma^\T.
\end{equation}
% where $\hat y_\gamma$ is the vectorization of $\{ \hat y_{\gamma,3,i}\}$, the $i$-th row of $\hat x_\gamma$ is $\hat x_{\gamma,3,i}$, and $\hat K_{ X_\gamma  X_\gamma} = \hat x_\gamma^\T \hat x_\gamma$.

\subsection{KAR Estimator}
Given observational data of size $N$, $\{(x_i, y_i, z_i)\}_{i \in [N]}$, the KAR procedure can be performed in two ways based on the two variants in the projection stage.
% \wk{mention randomised index for data splitting below}
\paragraph{Three-stage KAR}
To apply the disjoint sample sets projection in \Cref{sec:disjoint_proj}, we \emph{randomly} split the data set of size $N$ into three disjoint sets of sample size $n_1,n_2,m$ where $N=n_1 + n_2 + m$ 
%and index them accordingly 
and re-index them from $\{1:N\}$. The first two sets of data $\{(x_i, z_i)\}_{i \in \{1 : n_1\}}$ and  $\{(y_j, z_j)\}_{j \in \{n_1+1 : n_1+n_2\}}$ are used for learning the projection operators in \Cref{eq:est_proj_x1} and \Cref{eq:est_proj_y1}.{ 
We note that 
%in this stage
samples $\{y_i\}_{i\in\{1:n_1\}}$ and $\{x_j\}_{j\in\{n_1+1:n_1+n_2\}}$ are not used.}
The third set $\{(x_l, y_l, z_l)\}_{l \in \{n_1+n_2+1 : N\}}$ is used for regression stage to learn $ \widehat H_\xi^{\gamma,m}$ in \Cref{eq:KAR_estimator}. 
This procedure, termed \textbf{KAR}, includes solving three different regression problems,
% \footnote{However, our three-stage least square regression procedure is not be confused with the three-stage least square (3SLS) \citep{zellner1962three} under the framework of seemingly unrelated regressions (SUR) and solving simultaneous equation, where their extra stage from 2SLS is concatenating the simultaneous equations for variance reduction. Our third stage here is an additional kernel ridge regression from data splitting. Explicit connections of KAR with 3SLS are provided in the Appendix. }, 
which is different from the two-stage settings used in linear anchor regression \citep{rothenhausler2018anchor}.
% Our three-stage least squares algorithm requires sample splitting to alleviate the finite sample bias \citep{angrist1995split}.

\paragraph{Two-stage KAR}
For the joint sample set projection in \Cref{sec:joint_proj}, we only split the data of size $N$ into two disjoint sets \emph{randomly} of size $n$ and $m$ where $N = n + m$ and re-index them such that $\{(x_i, y_i, z_i)\}_{i \in \{1 : n\}}$ and $\{(x_l, y_l, z_l)\}_{l \in \{n+1 : N\}}$.
The first set is then grouped into $\{(x_i, z_i)\}_{i \in \{1 : n\}}$ and  $\{(y_i, z_i)\}_{j \in \{1:n\}}$ to learn the projection operators in \Cref{eq:est_proj_x1} and \Cref{eq:est_proj_y1}.{ 
In this manner, $\{z_i\}_{i\in\{1:n\}}$ are used twice.} 
The second set $\{(x_l, y_l, z_l)\}_{l \in \{n+1 : N\}}$ is used for regression stage to learn $ \widehat H_\xi^{\gamma,m}$ in \Cref{eq:KAR_estimator}, which is the same as the three-stage procedure above.  
This procedure, termed \textbf{KAR.2}, replicates the 2SLS used in KIV \citep{singh2019kernel} and linear anchor regression \citep{rothenhausler2018anchor}. 

% Consider a data set with $n_1 +n_2 +m$ observations of $(Z,X,Y)$, we denote $n_1$ stage I observations by $(z_{1,i}, x_{1,i}, y_{1,i})$, $n_2$ stage II observations by $(z_{2,i}, x_{2,i}, y_{2,i})$, and $m$ stage III observations by $(z_{3,i}, x_{3,i}, y_{3,i})$. Let $n = n_1 +n_2$ denote the number of observations from first two stages.
% The kernel anchor regression estimator can be computed through a three-stage algorithm. This algorithm is an extension to the classic two-stage least squares algorithm (2SLS) that are widely used in economics \citep{bollen1996alternative}.

\section{
% Theoretical 
Analysis of KAR Estimators}\label{sec:analysis}

\subsection{Consistency}
In this section, we first focus on the three-stage KAR procedure with disjoint sample sets for projection in \Cref{sec:disjoint_proj}.
% the consistency of kernel anchor regression estimator. 
The closed form solutions and convergence rates of the estimators are extended from the analysis of 2SLS in KIV \citep{singh2019kernel}. 
We follow the integral operator notations in \citet{singh2019kernel}. Define the 
% stage I and stage II 
projection stage operators as
\begin{eqnarray*}
S_1^* &:& \H_\Z \rightarrow L^2(\Z,\rho_\Z), \quad
 l \rightarrow \langle l, \phi(\cdot)\rangle_{\H_\Z},\\
S_1 &:& L^2(\Z,\rho_\Z) \rightarrow \H_\Z, \quad \tilde l \rightarrow \int \phi(z)\tilde l(z) d\rho_\Z(z),
\end{eqnarray*}
where $\rho$ denotes the joint distribution of $(Z, X, Y)$. $L^2(\Z,\rho_\Z)$ denotes the space of square integrable functions from $\Z$ to $\Y$ with respect to measure $\rho_\Z$, where $\rho_\Z$ is the restriction of $\rho$ to $\Z$. $T_1 = T_2 = S_1 \circ S_1^*$ are then uncentered covariance operators.
%We place assumptions on the original spaces $\X$, $\Z$ and $\Y$, the scalar-valued RKHSs $\H_\X$, $\H_\Z$, the adjusted spaces $\X_a$, $\Y_a$ and the probability distribution $\rho (x,z)$.
% 
We define the power of operator $T_1$ with respect to its eigendecomposition. Let $\H_\Gamma = \H_\X\otimes\H_\Z$, $\H_\Theta = \Y \otimes\H_\Z$ and $\H_\Omega = \Y \otimes \H_\X$ be the relevant tensor product spaces for the operators.

\begin{condition}\label{cond::s1}
% Suppose that
% \begin{itemize}
%     \item[\textup{(i)}] $\X$ and $\Z$ are Polish spaces, i.e. separable and completely metrizable topological spaces.
%     \item[\textup{(ii)}] $k_\X$ and $k_\Z$ are continuous and bounded: 
    
%     $\sup_{x \in \X} \Vert \psi(x) \Vert_{\H_\X} \leq Q_1$, $\sup_{z \in \Z} \Vert \phi(z) \Vert_{\H_\Z} \leq \kappa$.
%     \item[\textup{(iii)}] $\psi$ and $\phi$ are measurable.
%      %\citep{sriperumbudur2010relation}
%     \item[\textup{(iv)}] $k_\X$ is characteristic.
%     \item[\textup{(v)}] $E_{X} \in \H_{\Gamma}$, then 
    
%     $\E_{1,X}(E_{X}) = \inf_{E \in \H_{\Gamma}} \E_{1,X}(E)$.
%     \item[\textup{(vi)}] Fix $\zeta_1 < \infty$. For given $c_1 \in (1,2]$, define the prior $\mathcal P (\zeta_1, c_1)$ as the set of probability distributions $\rho$ on $\X \times \Z$ s.t. a range space assumption is satisfied: $\exists G_1 \in \H_\Gamma$ s.t. 
%     $E_{X} = T_1^{\frac{c_1-1}{2}} \circ G_1$ and $\Vert G_1 \Vert_{\H_\Gamma}^2 \leq \zeta_1$.
% \end{itemize}
    
% \item[\textup{(i)}] 
    (i) $\X$ and $\Z$ are Polish, i.e. separable and completely metrizable topological spaces.
% \item[\textup{(ii)}] 
(ii) $k_\X$ and $k_\Z$ are continuous and bounded: 
    $\sup_{x \in \X} \Vert \psi(x) \Vert_{\H_\X} \leq Q_1$, $\sup_{z \in \Z} \Vert \phi(z) \Vert_{\H_\Z} \leq \kappa$.
    % \item[\textup{(iii)}] 
    (iii) $\psi$ and $\phi$ are measurable.
     %\citep{sriperumbudur2010relation}
    % \item[\textup{(iv)}] 
    (iv) $k_\X$ is characteristic.
    % \item[\textup{(v)}] 
    (v) $E_{X}^p \in \H_{\Gamma}$ s.t. 
    % then 
    $\E(E_{X}^p) = \inf_{E_X \in \H_{\Gamma}} \E(E_{X})$.
    % \item[\textup{(vi)}] 
    (vi) Fix $\zeta_1 < \infty$. For 
    % given 
    $c_1 \in (1,2]$, define the prior $\mathcal P (\zeta_1, c_1)$ as the set of probability distributions $\rho$ on $\X \times \Z$ s.t.
    % a range space assumption is satisfied: 
    $\exists G_1 \in \H_\Gamma$ s.t. 
    $E_{X}^p = T_1^{(c_1-1)/2} \circ G_1$ and $\Vert G_1 \Vert_{\H_\Gamma}^2 \leq \zeta_1$.
\end{condition}

Condition~\ref{cond::s1} 
% below 
% has been proposed 
is adapted from \citet{singh2019kernel} to bound the approximation error of the regularized estimator $E_{\alpha_1,X}^{n_1}$. 
Parameter $c_1$ suggests the smoothness of conditional 
% expectation 
operator $E_{\alpha_1,X}^{n_1}$. A larger $c_1$ corresponds to a smoother operator.

\begin{lemma}\label{lem::s1}
$\forall \alpha_1 > 0$, the solution $E_{\alpha_1,X}^{n_1}$ of the regularized empirical objective 
%$\E_{\alpha_1, X}^{n_1}$ 
in \Cref{eq:proj_emp1} exists and is unique. With $
    \mathbf{T}_1 = \frac{1}{n_1} \sumni \phi(z_{i}) \otimes \phi(z_{i})$ and
    $\mathbf{g}_1 = \frac{1}{n_1} \sumni \phi(z_{i}) \otimes \psi(x_{i})$, the estimator in \Cref{eq:est_proj_x1} has the form
    $    E_{\alpha_1,X}^{n_1} = (\mathbf{T}_1 + \alpha_1)^{-1} \circ \mathbf{g}_1.$
% \begin{eqnarray*}
%     E_{\alpha_1,X}^{n_1} &=& (\mathbf{T}_1 + \alpha_1)^{-1} \circ \mathbf{g}_1.
% \end{eqnarray*}
% \wk{This equation about is for introducing alternative formulation of just connecting to $T_1$?}

Under Condition~\ref{cond::s1} and $\alpha_1  = n_1^{-1/(c_1+1)}$, we have: 
\begin{eqnarray*}
    \Vert E_{\alpha_1,X}^{n_1} - E_{X}^p \Vert_{\H_\Gamma} = O_p(n_1^{-\frac{c_1-1}{2(c_1+1)}}).
\end{eqnarray*}
% when $\alpha_1  = n_1^{-1/(c_1+1)}$.
\end{lemma}

% \begin{lemma}\label{lem::s1}
% $\forall \alpha_1 > 0$, the solution $E_{\alpha_1,X}^{n_1}$ of the regularized empirical objective $\E_{\alpha_1, X}^{n_1}$ exists, is unique, and
% \begin{eqnarray*}
%     E_{\alpha_1,X}^{n_1} &=& (\mathbf{T}_1 + \alpha_1)^{-1} \circ \mathbf{g}_1, \\
%     \mathbf{T}_1 &=& \frac{1}{n_1} \sumni \phi(z_{1,i}) \otimes \phi(z_{1,i}), \\
%     \mathbf{g}_1 &=& \frac{1}{n_1} \sumni \phi(z_{1,i}) \otimes \psi(x_{1,i}).
% \end{eqnarray*}
% Under Condition~\ref{cond::s1}, $\forall \delta \in (0,1)$, the following holds w.p. $1 - \delta$:
% \begin{eqnarray*}
%     \Vert E_{\alpha_1,X}^{n_1} - E_{X} \Vert_{\H_\Gamma} \leq
%     r_{E_1}(\delta,n_1,c_1) := \\
%     \frac{ \sqrt{\zeta_1} (c_1 + 1)}{4^{\frac{1}{c_1+1}}} \left( \frac{4\kappa (Q_1 + \kappa \Vert E_{X} \Vert_{\H_{\Gamma}} \ln(2/\delta) }{\sqrt{n_1 \zeta_1}(c_1 -1)} \right)^{\frac{c_1-1}{c_1+1}},\\
%     \alpha_1 = \left( \frac{8\kappa (Q_1 + \kappa \Vert E_{X} \Vert_{\H_\Gamma} \ln(2/\delta) }{\sqrt{n_1 \zeta_1}(c_1 -1)} \right)^{\frac{2}{c_1+1}}.
% \end{eqnarray*}
% \end{lemma}

Lemma~\ref{lem::s1} follows from \citet{singh2019kernel}, and shows that the efficient rate of $\alpha_1$ is $n_1^{-1/(1+c_1)}$. Note that the convergence rate of $E_{\alpha_1,X}^{n_1}$ is calibrated by $c_1$, which measures the smoothness of the conditional expectation operator $E_{X}$.

For the disjoint set projection in \Cref{sec:disjoint_proj},
the closed form solution and convergence rate for learning $P_{\z}Y$ estimator is similar to that of learning $P_\z \psi(X)$ due to the independent estimation procedure and further requires the following conditions.

\begin{condition}\label{cond::s2}
% Suppose that
% \begin{itemize}
    % \item[\textup{(i)}] $\Y$ is Polish space.
    % \item[\textup{(ii)}] $Y$ is bounded: $\sup_{y \in \Y} \Vert y \Vert_{\Y} \leq Q_2$.
    % \item[\textup{(iii)}] $E_{Y} \in \H_{\Theta}$, then 
    
    % $\E_{2,Y}(E_{ Y}) = \inf_{E \in \H_{\Theta}} \E_{2,Y}(E)$. 
    % \item[\textup{(iv)}] Fix $\zeta_2 < \infty$. For given $c_2 \in (1,2]$, define the prior $\mathcal P (\zeta_2, c_2)$ as the set of probability distributions $\rho$ on $\Y \times \Z$ s.t. a range space assumption is satisfied: $\exists G_2 \in \H_\Theta$ s.t. 
    % $E_{Y} = T_2^{\frac{c_2-1}{2}} \circ G_2$ and $\Vert G_2 \Vert_{\H_\Theta}^2 \leq \zeta_2$.
% \end{itemize}
    \textup{(i)} $\Y$ is a Polish space.
    \textup{(ii)} $Y$ is bounded: $\sup_{y \in \Y} \Vert y \Vert_{\Y} \leq Q_2$.
    \textup{(iii)} $E_{Y}^p \in \H_{\Theta}$ s.t. 
    $\E(E_{Y}^p) = \inf_{E_Y \in \H_{\Theta}} \E(E_Y)$.
    \textup{(iv)} Fix $\zeta_2 < \infty$. For 
    % given 
    $c_2 \in (1,2]$, define the prior $\mathcal P (\zeta_2, c_2)$ as the set of probability distributions $\rho$ on $\Y \times \Z$ s.t. 
    % a range space assumption is satisfied: 
    $\exists G_2 \in \H_\Theta$ s.t. 
    $E_{Y}^p = T_2^{(c_2-1)/{2}} \circ G_2$ and $\Vert G_2 \Vert_{\H_\Theta}^2 \leq \zeta_2$.
\end{condition}

\begin{lemma}\label{lem:s2}
$\forall \alpha_2 > 0$, the solution $E_{\alpha_2,Y}^{n_2}$ of the regularized empirical objective 
%$\E_{\alpha_2,Y}^{n_2}$ 
in \Cref{eq:proj_emp2} exists and is unique. With $\mathbf{T}_2 = \frac{1}{n_2} \sumnj \phi(z_{j}) \otimes \phi(z_{j})$ and
    $\mathbf{g}_2 = \frac{1}{n_2} \sumnj \phi(z_{j}) y_{j}$, the estimator in \Cref{eq:est_proj_x1} has the form
    $    E_{\alpha_2,Y}^{n_2} = (\mathbf{T}_2 + \alpha_2)^{-1} \circ \mathbf{g}_2.
    $
% \begin{eqnarray*}
%     E_{\alpha_2,Y}^{n_2} &=& (\mathbf{T}_2 + \alpha_2)^{-1} \circ \mathbf{g}_2.
%     % \mathbf{T}_2 &=& \frac{1}{n_2} \sumni \phi(z_{2,i}) \otimes \phi(z_{2,i}), \\
%     % \mathbf{g}_2 &=& \frac{1}{n_2} \sumni \phi(z_{2,i}) y_{2,i}.
% \end{eqnarray*}
Under Condition~\ref{cond::s1}--~\ref{cond::s2} and $\alpha_2  = n_2^{-1/(c_2+1)}$, we have:
\begin{eqnarray*}
    \Vert E_{\alpha_2,Y}^{n_2} - E_{Y}^p \Vert_{\H_\Theta} = O_p(n_2^{-\frac{c_2-1}{2(c_2+1)}}).
\end{eqnarray*}
% when $\alpha_2  = n_2^{-1/(c_2+1)}$.
\end{lemma}

% \begin{lemma}\label{lem:s2}
% $\forall \alpha_2 > 0$, the solution $E_{\alpha_2,Y}^{n_2}$ of the regularized empirical objective $\E_{\alpha_2}^{n_2}$ exists, is unique, and
% \begin{eqnarray*}
%     E_{\alpha_2,Y}^{n_2} &=& (\mathbf{T}_2 + \alpha_2)^{-1} \circ \mathbf{g}_2, \\
%     \mathbf{T}_2 &=& \frac{1}{n_2} \sumni \phi(z_{2,i}) \otimes \phi(z_{2,i}), \\
%     \mathbf{g}_2 &=& \frac{1}{n_2} \sumni \phi(z_{2,i}) y_{2,i}.
% \end{eqnarray*}
% Under Condition~\ref{cond::s1} and Condition~\ref{cond::s2}, $\forall \epsilon \in (0,1)$, the following holds w.p. $1 - \epsilon$:
% \begin{eqnarray*}
%     \Vert E_{\alpha_2,Y}^{n_2} - E_{Y} \Vert_{\H_\Theta} \leq
%     r_{E_2}(\epsilon,n_2,c_2) := \\
%     \frac{ \sqrt{\zeta_2} (c_2 + 1)}{4^{\frac{1}{c_2+1}}} \left( \frac{4\kappa (Q_2 + \kappa \Vert E_{Y} \Vert_{\H_{\Theta}} \ln(2/\epsilon) }{\sqrt{n_2 \zeta_2}(c_2 -1)} \right)^{\frac{c_2-1}{c_2+1}},\\
%     \alpha_2 = \left( \frac{8\kappa (Q_2 + \kappa \Vert E_{Y} \Vert_{\H_\Theta} \ln(2/\epsilon) }{\sqrt{n_2 \zeta_2}(c_2 -1)} \right)^{\frac{2}{c_2+1}}.
% \end{eqnarray*}
% \end{lemma}

Similar to learning projection $P_\z \psi(X)$, the efficient rate of $\alpha_2$ is $n_2^{-1/(1+c_2)}$, where $c_2$ measures the smoothness of the conditional expectation operator $E_{Y}$.

Let $L^2(\H_\X,\rho_{\H_\X})$ denote the space of square integrable functions from $\H_\X$ to $\Y$ with respect to measure $\rho_{\H_\X}$, where $\rho_{\H_\X}$ is the extension of $\rho$ to $\H_\X$. Define the regression stage operator as %\td{$X_{\gamma}$ is changed to $\psi_{\gamma}$ or  $\psi_{\gamma}(X)$; need to unify here and Thm 1. }
%{\wenqi $X_{\gamma}$ and $X_a$ are changed to $\psi_{\gamma}$.}
\begin{eqnarray*}
S^* &:& \H_\Omega \rightarrow L^2(\H_\X,\rho_{\H_\X}), \quad 
H \rightarrow \Omega^*_{(\cdot)} H,\\
S &:& L^2(\H_\X,\rho_{\H_\X}) \rightarrow \H_\Omega, \\ 
&&\tilde H \rightarrow \int \Omega_{\psi_\gamma} \circ \tilde H \psi_\gamma d \rho_{\H_\X}(\psi_\gamma),
\end{eqnarray*}
where $\Omega_{\psi_\gamma}: \Y \rightarrow \H_\Omega$ defined by $y \rightarrow \Omega (\cdot, \psi_\gamma)y$ is the point evaluator of  \citet{micchelli2005learning}. 
% Then d
Define $T_{\psi_\gamma} = \Omega_{\psi_\gamma} \circ \Omega^*_{\psi_\gamma}$ and covariance operator $T=S \circ S^*$. Define the power of operator $T$ with respect to its eigendecomposition. Condition~\ref{cond::s3} below extends hypothesis 7--9 in \citet{singh2019kernel}, and is sufficient to bound the excess error of 
%KAR estimator 
$\widehat H_{\xi}^{\gamma, m}$ with 
% the convergence rate of 
the error propagated from the estimators in the projection stage. 
% I and stage II estimators.

\begin{condition}\label{cond::s3}
% Suppose that
\begin{itemize}
    \item[\textup{(i)}] The $\{ \Omega_{\psi_\gamma}\}$ operator family is uniformly bounded in Hilbert-Schmidt norm: $\exists B$ s.t. $\forall \psi_\gamma$, 
    %$\Vert \Omega_{\psi_\gamma} \Vert_{HS}^2 = Tr( \Omega^*_{\psi_\gamma} \circ \Omega_{\psi_\gamma}) \leq B$.
    $\Vert \Omega_{\psi_\gamma} \Vert_{\L_2(\Y,\H_\Omega)}^2 = Tr( \Omega^*_{\psi_\gamma} \circ \Omega_{\psi_\gamma}) \leq B$. 
    \item[\textup{(ii)}] The $\{ \Omega_{\psi_\gamma}\}$ operator family is is Hölder continuous in operator norm: $\exists L > 0, \iota \in (0,1]$ s.t. $\forall \psi_\gamma, \psi_\gamma^\prime, \Vert \Omega_{\psi_\gamma} - \Omega_{\psi_\gamma^\prime} \Vert_{\L(\Y,\H_\Omega)} \leq L \Vert \psi_\gamma - \psi_\gamma^\prime \Vert_{\H_\X}^\iota$.
    \item[\textup{(iii)}] $H^\gamma \in \H_\Omega$, then $\E^\gamma(H^\gamma) = \inf_{H \in \H_\Omega} \E^\gamma(H)$.
    \item[\textup{(iv)}] $Y_\gamma$ is bounded, i.e. $\exists C < \infty$ s.t. $\Vert Y_\gamma \Vert_\Y \leq C$.
    % almost surely.
    \item[\textup{(v)}] Fix $\zeta < \infty$. For given $b_\gamma \in (1,\infty]$ and $c_\gamma \in (1,2]$, define the prior $\mathcal P (\zeta, b_\gamma, c_\gamma)$ as the set of probability distributions $\rho$ on $\H_\X \times \Y$ s.t. 
    \begin{itemize}
        \item [\textup{(a)}] 
        range space assumption is satisfied: 
        $\exists G \in \H_\Omega$ s.t. $H^\gamma = T^{\frac{(c_\gamma-1)}{2}} \circ G$ and $\Vert G \Vert_{\H_\Omega}^2 \leq \zeta$;
        \item [\textup{(b)}] the eigenvalues from spectral decomposition $T = \sum_{k=1}^\infty \lambda_k e_k\langle\cdot, e_k\rangle_{\H_\Omega} $, where $\{e_k\}_{k=1}^\infty$ is a basis of $Ker(T)^{\bot}$
        % \footnote{Ker(T)=\{v|Tv=0\}}
        , 
        % the eigenvalues 
        satisfy $\alpha \leq k^{b_\gamma}\lambda_k \leq \beta$ for some $\alpha, \beta >0$.
        %\wk{what is $\gamma_k$ here? or it's $\lambda_k$. }
    \end{itemize}
\end{itemize}
\end{condition}
We {note that 
%  since $\Omega_{\psi_\gamma}$ and $Y^\gamma$ depend on $\gamma$, 
all parameters mentioned in Condition~\ref{cond::s3} 
% are actually functions of 
depend on $\gamma$, though the function representations are not explicit. We set subscript $\gamma$ especially for parameters $b_\gamma$ and $c_\gamma$ to emphasize their dependency on $\gamma$.} Parameter $b_\gamma$ measures the decay of eigenvalues of the covariance operator $T$ and specifically, larger $b_\gamma$ suggests smaller effective input dimension. A larger $c_\gamma$ corresponds to a smoother operator $H^\gamma$.

% The estimator in stage III has a closed form as follows.
\begin{lemma}\label{lem:s3}
$\forall \xi > 0$, 
%the solution $E_{\xi}^{\gamma,m}$ to $\E_\xi^{\gamma,m}$ and 
the solution $\widehat H_\xi^{\gamma,m}$ to $\widehat \E_\xi^{\gamma,m}$ exist, and is unique for each $\gamma$. 
%{\wenqi The current definition of $\E_\xi^{\gamma,m}$ is actually $\widehat \E_\xi^{\gamma,m}$; missing definition for $\E_\xi^{\gamma,m}$.} 
%\wk{yes, amended above. so i m thinking to define $\E_\xi^{\gamma,m}$ in the appendix for the proof only. or is there a particular place in main text we need $\E_\xi^{\gamma,m}$?}
%{\wenqi I see. I delete $\E_\xi^{\gamma,m}$ and the corresponding terms in the main text.}
% \begin{eqnarray*}
%     %&\mathbf{T}= \frac{1}{m} \summ T_{\psi_{\gamma,l}}, \quad
%     %\mathbf{g} = \frac{1}{m} \summ \Omega_{\psi_{\gamma,l}} y_{\gamma,l},\\
%     %&H_\xi^{\gamma,m} = (\mathbf{T} + \xi)^{-1} \circ \mathbf{g}, \\
%     &\widehat{\mathbf{T}}= \frac{1}{m} \summ T_{\widehat \psi_{\gamma,l}}, \quad
%     \widehat{\mathbf{g}} = \frac{1}{m} \summ \Omega_{\widehat \psi_{\gamma,l}} \widehat y_{\gamma,l},\\
%     &\widehat H_\xi^{\gamma,m} = ( \hat{\mathbf{T}} + \xi)^{-1} \circ \hat{\mathbf{g}}.
% \end{eqnarray*}
Let $\widehat{\mathbf{T}}= \frac{1}{m} \summ T_{\widehat \psi_{\gamma,l}}$, $\widehat{\mathbf{g}} = \frac{1}{m} \summ \Omega_{\widehat \psi_{\gamma,l}} \widehat y_{\gamma,l}$. \Cref{eq:KAR_estimator} has the form
$$\widehat H_\xi^{\gamma,m} = ( \widehat{\mathbf{T}} + \xi)^{-1} \circ \widehat{\mathbf{g}}.$$
\end{lemma}

\begin{condition}
\label{cond::s1&2}
For $c_1, c_2$ 
% {\wenqi(I delete $\iota$ because it cannot be set)}
% and $\iota$ 
set in Conditions \ref{cond::s1}--~\ref{cond::s2} and $\iota$ satisfying Condition~\ref{cond::s3}, assume
$n_2 \geq n_1^{\frac{\iota(c_1-1)(c_2+1)}{(c_1+1)(c_2-1)}}$.
\end{condition}

\begin{remark}
Condition~\ref{cond::s1&2} is sufficient but not necessary to ensure that the error propagates to regression stage from estimating $E_Y^p$ is smaller than that from estimating $E_X^p$ in disjoint sample sets projection.
% the error estimating $E_Y$ is smaller than estimating $E_X$ and propagate to regression stage.
% {\wenqi I change the remark, but it's still a little bit weird.}
\end{remark}

The main challenge of extending the convergence rate of 
% kernel instrumental variable 
KIV estimator \citep{singh2019kernel} to 
% kernel anchor regression 
KAR estimator is that in our case, the excess error depends not only on the accuracy of $E_X^p$ estimator but also on the accuracy of $E_Y^p$ estimator.
% , while the latter source of error does not exist in kernel instrumental variable case. 
However, by assuming Condition~\ref{cond::s1&2}, we ensure the error 
% caused by
from estimating $E_Y^p$ is dominated by that of $E_X^p$, and manage to illustrate the optimal convergence rate for KAR
% kernel anchor regression 
as shown below in Thereom~\ref{thm::s3c}.
% {\wenqi This paragraph is shortened.}

%We require that the error propagated from estimating $E_Y$ to be dominated by that from estimating $E_X$ in the projection stage, because the estimation error of $E_X$ influences the final excess error in a more complex way. Intuitively, $\Vert \Omega_{\widehat \psi_\gamma} - \Omega_{\psi_\gamma} \Vert$ contributes to both $\Vert \hat{\mathbf{g}} -  \mathbf{g} \Vert$ and $\Vert \hat{\mathbf{T}} -  \mathbf{T} \Vert$, while $\Vert \widehat y_{\gamma} - y_{\gamma} \Vert$ only contributes to $\Vert \hat{\mathbf{g}} -  \mathbf{g} \Vert$. The exact components of excess error are given in supplementary material.

%The main challenge of extending the convergence rate of kernel instrumental variable estimator to kernel anchor regression estimator is that the error $\Vert \hat{\mathbf{g}} -  \mathbf{g} \Vert$ is now a compound of $\Vert \Omega_{\widehat \psi_\gamma} - \Omega_{\psi_\gamma} \Vert$ caused by estimating $E_X$ and $\Vert \widehat y_{\gamma} - y_{\gamma} \Vert$ caused by estimating $E_Y$, while the excess error of kernel instrumental variable estimator in \citet{singh2019kernel} only contains the first term. However, we find that by assuming Condition~\ref{cond::s1&2}, the error caused by estimating $E_X$ dominates the term $\Vert \hat{\mathbf{g}} -  \mathbf{g} \Vert$. As Thereom~\ref{thm::s3c} below shows, the convergence rate of $\hat H_\xi^{\gamma,m}$ is then identical to the convergence rate of Kernel IV estimator from \citep{singh2019kernel}.

\begin{theorem}
\label{thm::s3c}
Under Condition~\ref{cond::s1}--~\ref{cond::s1&2}, let $d_1, d_2 > 0$
and 
choose
$\alpha_1 = n_1^{-\frac{1}{c_1+1}}$, $\alpha_2 = n_2^{-\frac{1}{c_2+1}}$,
    $n_1 = m^{\frac{d_1(c_1+1)}{\iota(c_1-1)}}$, 
    $n_2 = m^{\frac{d_2(c_2+1)}{\iota(c_2-1)}}$,
% \begin{eqnarray*}
%     \alpha_1 = n_1^{-\frac{1}{c_1+1}}, \quad
%     \alpha_2 = n_2^{-\frac{1}{c_2+1}},\\
%     n_1 = m^{\frac{d_1(c_1+1)}{\iota(c_1-1)}}, \quad
%     n_2 = m^{\frac{d_2(c_2+1)}{\iota(c_2-1)}},
% \end{eqnarray*}
% $\alpha_1 = n_1^{-\frac{1}{c_1+1}}$, $\alpha_2 = n_2^{-\frac{1}{c_2+1}}$, $n_1 = m^{\frac{d_1(c_1+1)}{\iota(c_1-1)}}$ and $n_2 = m^{\frac{d_2(c_2+1)}{\iota(c_2-1)}}$, 
% where $d_1, d_2 > 0$. 
We have:
\begin{itemize}
    \item [\textup(i)] If $d_1 \leq \frac{b_\gamma(c_\gamma+1)}{b_\gamma c_\gamma+1}$, then $\E^\gamma(\widehat H_\xi^{\gamma,m}) - \E^\gamma( H^\gamma) = \op(m^{-\frac{d_1 c_\gamma}{c_\gamma+1}})$ with $\xi = m^{-\frac{d_1}{c_\gamma+1}}$.
    \item [\textup(ii)] If $d_1 > \frac{b_\gamma(c_\gamma+1)}{b_\gamma c_\gamma+1}$, then $\E^\gamma(\widehat H_\xi^{\gamma,m}) - \E^\gamma( H^\gamma) = \op(m^{-\frac{b_\gamma c_\gamma}{b_\gamma c_\gamma+1}})$ with $\xi = m^{-\frac{b_\gamma}{b_\gamma c_\gamma+1}}$.
\end{itemize}
\end{theorem}
At $d_1 = b_\gamma(c_\gamma+1)/(b_\gamma c_\gamma+1) < 2$, the convergence rate of KAR estimator $m^{-b_\gamma c_\gamma/(b_\gamma c_\gamma+1)}$ is optimal. 
%while requiring the fewest observations. 
This statistically efficient rate is calibrated by $b_\gamma$, the effective input dimension, 
% as well as
together with $c_\gamma$, the smoothness of 
% structural 
the operator $H^\gamma$. The condition $d_1 = b_\gamma(c_\gamma+1)/(b_\gamma c_\gamma+1) < 2$ also suggests that $n_1 > m$.

\iffalse
\begin{corollary}[two-stage]
Let ... then $\|H_{\xi}^{\gamma, m} - H \|_{\H_{\Omega}}\to 0$ as $n,m\to \infty$.
\end{corollary}
\td{Then comment on 2-stage 3-stage difference}
\fi
Additional results 
% of \textbf{KAR.2} and comparisons with \textbf{KAR} are 
and discussions 
including the two-stage approach 
are included in the Appendix.

\subsection{Causal effect and target KAR estimate}
In this section, we discuss the scenarios assuming that the data are generated from a 
% kernelized 
structural causal model with nonlinear features 
as shown below,
\begin{align}\label{eq:nonlinear_sem}
    % \left(
    \begin{pmatrix}
        C  \\
        \psi(X) \\
        Y \end{pmatrix}
        % \right)
        = B
%   \left(
   \begin{pmatrix}
        \phi(Z) \\
        C  \\
        \psi(X) \\
        Y \end{pmatrix}
        % \right)
        + 
        % \left(
        \begin{pmatrix}
        \epsilon_C  \\
        \epsilon_X \\
        \epsilon_Y\end{pmatrix},
        % \right), 
\end{align}
where we write operator $B$ in the following matrix form
% $\phi(Z) = \epsilon_z$, and
% $$
% B = \left( \begin{array}{cccc}
%         B_{ZC} & 0 & 0 & 0 \\
%         B_{ZX} & B_{CX} & 0 & 0 \\
%         B_{ZY} & B_{CY} & B_{XY} & 0
%   \end{array} \right)
% $$
$$
B = \begin{pmatrix}
        B_{CZ} & 0 & 0 & 0 \\
        B_{XZ} & B_{XC} & 0 & 0 \\
        B_{YZ} & B_{YC} & B_{YX} & 0
   \end{pmatrix}.
$$
We note that each operator $B_{\triangle \square}$ represents an operator that takes an element from $\square$-related space to $\triangle$-related space, e.g. $B_{XZ}:\H_\Z \to \H_\X$ and $B_{YZ}:\H_\Z \to \Y$.
% 
% is an unknown constant matrix, the anchors $Z$ and the hidden variables $C$ are random vectors, 
The 
% (vector-valued)
noise variables $\epsilon_Z$, $\epsilon_C$, $\epsilon_X$ and $\epsilon_Y$ 
% random vectors that 
are independent of each other. Let $\Sigma_Z$, $\Sigma_C$, $\Sigma_X$ and $\Sigma_Y$ denote the covariance of $\epsilon_Z$, $\epsilon_C$, $\epsilon_X$ and $\epsilon_Y$, respectively. Here each 
% sub-matrix 
operator in $B$ represents a line in the model shown in Figure~\ref{fig:KAR}. For instance, $B_{CZ}$ stands for the line from $\H_\Z$ to $C$; $B_{YX}$ corresponds to the line from $\H_\X$ to $Y$. $B_{YX}$ in \Cref{eq:nonlinear_sem} reflects the causal effect we are interested in.
% \wk{The $\phi(Z)$ and $\psi(Y)$ here are set to be finite dimensional vector, so that $B_{ZC}$ and $B_{XY}$  are finite dim matrices instead of operators. The proof in draft are all done in linear $Z$ form. This is not a problem, but I wounder setting $B_{ZC}:\H_\Z \to \mathcal{C}$ as and operator, where $\mathcal C = \R$. Then, we need to define zero operator, e.g. $B_{ZC}\phi(z)=0,\forall z$ to amend the proof.}
% \wk{Another thing I wanna suggest is use $B_{CZ}$ instead of $B_{ZC}$, which naturally implying mapping from $Z$-related space to $C$ related space. It provides better intuition for derivations in the proof you will see.}
We study the identifiability scenarios where operator $B_{YX}$ can be learned via 
KAR
estimator $H^{\gamma}$.
%{\wenqi I'm confused whether we can let $B$ denote operator.}
\begin{theorem}
\label{thm::causal} 
An operator $B_{XZ}$ is a zero operator written by $B_{XZ}=0$, if $\langle \psi(x), B_{XZ}\phi(z)\rangle_{\H_\X}=0$, $\forall \psi(X) \in \H_\X, \phi(z) \in \H_\Z$.
Operator $B_{CZ} = 0$ if $c^\T B_{CZ}\phi(z)=0$, $\forall c \in \mathcal C, \phi(z) \in \H_\Z$. A matrix-valued operator, e.g. $B_{YC}=0$ if all entries are $0$.
For data generation process following \Cref{eq:nonlinear_sem}, we have $H^\gamma = B_{YX}$ in following cases.
\begin{itemize}
\vspace{-0.35cm}
    \item [\textup{(i)}] $B_{YC}=0$ and $\gamma = 0$, 
    i.e.
    % which suggests 
    % the case with 
    no 
    % unobserved 
    latent confounder.
    \item [\textup{(ii)}] $B_{YZ} + B_{YC}B_{CZ} = 0$ and $\gamma = \infty$, where kernel IV  is a special case, i.e. both $B_{YZ} = 0$ and  $B_{CZ} = 0$.
    \item [\textup{(iii)}] $B_{YC}=0$, $B_{YZ} + B_{YC}B_{CZ} = 0$, and $\gamma \geq 0$.
%The bias $\Vert H^\gamma - B_{XY} \Vert_2^2$ reaches its minimum in following cases:
    %\item [\textup{(iv)}] $\Sigma_{XY}^\para = a \Sigma_{XY}^\res$ for some $a > 0$, and $\gamma = \infty$.
    \item [\textup{(iv)}] $\Sigma_{YX}^\para = - a \Sigma_{YX}^\res$ for some $a > 0$, and $\gamma = 1/a$.
    
    $\Sigma_{YX}^\para = (B_{YZ} + B_{YC}B_{CZ})\Sigma_Z (B_{ZX} + B_{ZC}B_{CX}) $ denotes the covariance between $\psi(X)$ and $Y$ projected on the linear span from the components of $\phi(Z)$; and $\Sigma_{YX}^\res = B_{YC}\Sigma_{C}B_{CX}$ denote the covariance between the residuals of $\psi(X)$ and $Y$.
\vspace{-0.cm}
\end{itemize}
\end{theorem}
Thereom~\ref{thm::causal} (i) 
% above 
suggests that KPA 
% estimator 
is optimal 
% if assume 
when there is no unobserved confounder; (ii) is a generalized condition including
% kernel 
KIV; 
% By Thereom~\ref{thm::causal}
(iii) shows the KAR estimator identifies the causal relation from $X$ to $Y$ regardless of $\gamma$ with generalized KIV condition in (ii) and no latent confounder in (i);
% when both assumptions on no unobserved confounding and valid instrumental variable is true, 
% By Thereom~\ref{thm::causal}(iv), if
(iv) shows the KAR identifiability condition with appropriate choice of $\gamma$ when $\Sigma_{XY}^\para$ and $\Sigma_{XY}^\res$ 
%are in the same direction, kernel IV estimator is optimal in terms of bias; if $\Sigma_{XY}^\para$ and $\Sigma_{XY}^\res$ 
are in the 
% opposite 
flipped direction.
% KAR estimator with appropriate choice of $\gamma$ performs better than both kernel IV or kernel PA estimator. 
% \wk{what did you mean by perform better?}
In the next section, we show the empirical results for KAR and relevant baseline methods.

\section{  Empirical Results}\label{sec:simulation}

\subsection{Synthetic experiments}\label{sec:synthetic}
We consider the data generating process
of the following nonlinear structural equation,
% We conduct simulation to evaluate kernel anchor regression. Set $n_1 = 250$, $n_2 = 250$, $m = 200$, $n=n_1+n_2 = 500$, $\alpha_1 = 1.5n_1^{-0.5}$, $\alpha_2 = 1.5n_2^{-0.5}$, $\alpha = 1.5n^{-0.5}$, and $\xi = 1.5m^{-0.5}$. The structural model is set as follows,
\begin{equation*}
    Y = 0.75C - 0.25Z + \ln(|16X - 8| + 1) sgn (X - 0.5),
\end{equation*}
where $sgn(x)\in \{-1, 0, +1\}$ denotes the sign of $x$.
The explanatory variables $X, Z, C$ are generated from 
% the a joint normal distribution followed by the transformations as follows.
\begin{eqnarray*}
% \left( 
\begin{pmatrix}
        C  \\
        V \\
        W \end{pmatrix} 
        % \right) 
        \sim
        N \left( 
        % \left( 
        \begin{pmatrix}
        0  \\
        0 \\
        0 \end{pmatrix}, 
        % \right), 
        % \left( 
        \begin{pmatrix}
        1, 0.3, 0.2  \\
        0.3, 1, 0 \\
        0.2, 0, 1 \end{pmatrix} 
        % \right)
        \right),\\
    X = F \left( \frac{W+V}{\sqrt{2}} \right),\quad
    Z = F(W) - 0.5,
\end{eqnarray*}
where $F$ denote the c.d.f of standard normal distribution. 

We generate $\{(x_i,y_i,z_i)\}_{i\in[N]}$ samples with $N=700$. To perform the data-splitting procedures described in \Cref{sec:kar}, we set $n_1 = n_2 = 250 $ and $n=500$ for a fair comparison in the projection stage (\Cref{sec:projection}); and $m=200$ in the regression stage (\Cref{sec:regression}). We set regularizers as
$\alpha_1 = 1.5n_1^{-0.5}$, $\alpha_2 = 1.5n_2^{-0.5}$, $\alpha = 1.5n^{-0.5}$ and $\xi = 1.5m^{-0.5}$.

\begin{figure}[t!]
    \centering
    % \includegraphics[width=0.45\textwidth,height=0.4\textwidth]{aistats23/Sim_fit_Combined.pdf}
    % \subfigure[Fitting: non-linear models]
    {\includegraphics[width=0.4\textwidth]{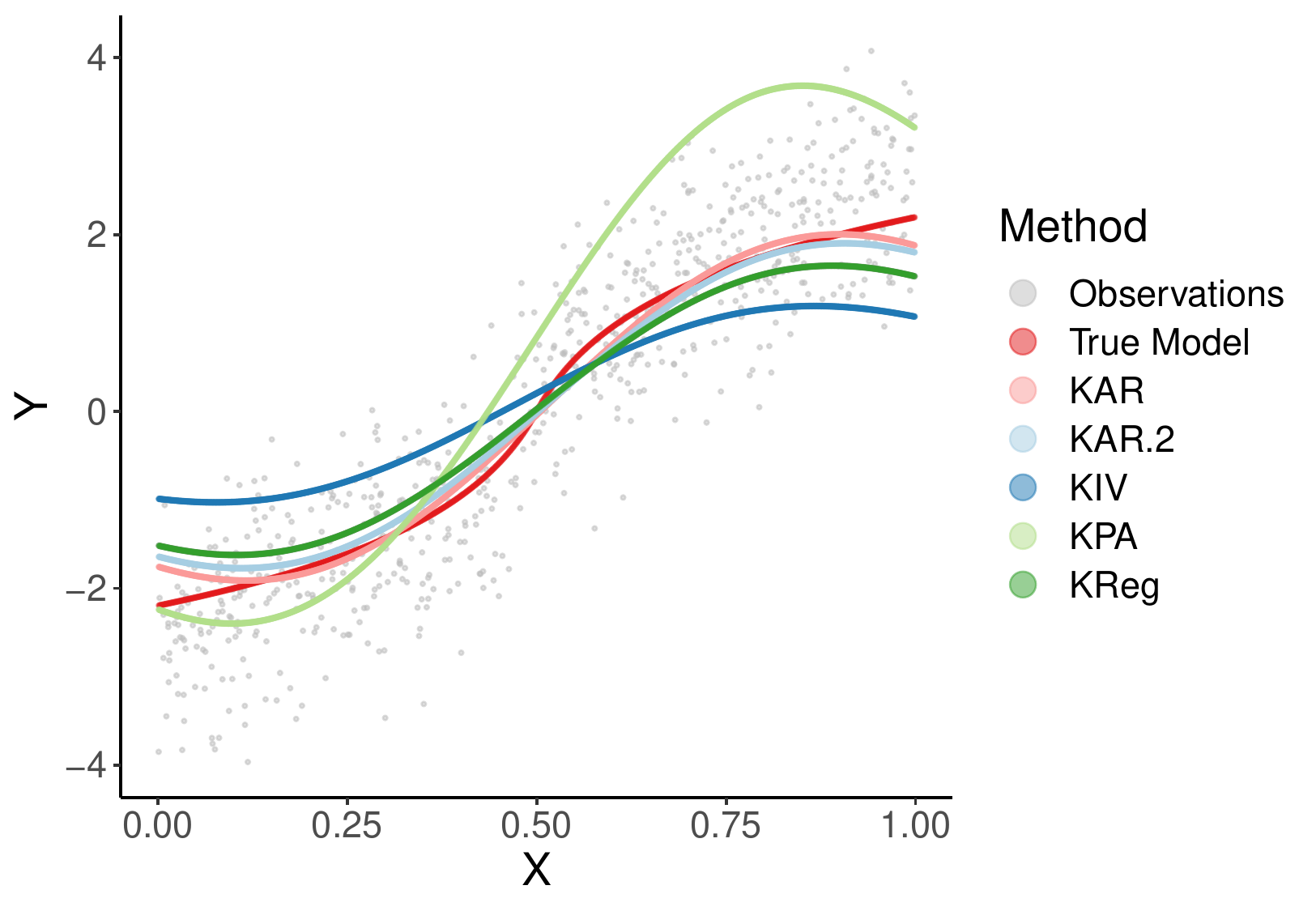}}
    % \subfigure[Fitting: linear models]
    {
    \includegraphics[width=0.4\textwidth]{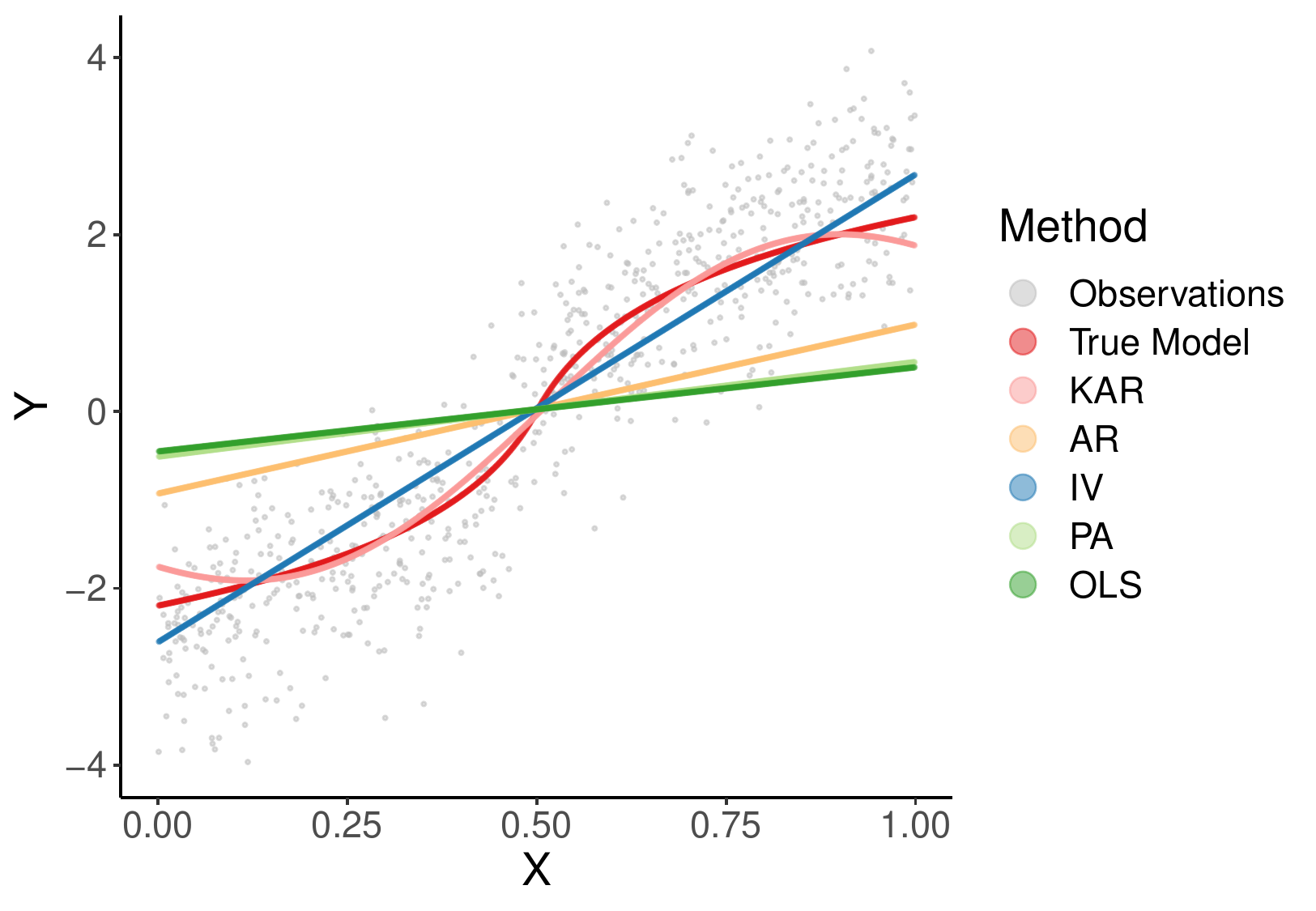}}    
    
    % \subfigure[log of MSE]
    {\includegraphics[width=0.56\textwidth
    % ,height=0.18\textwidth
    ]{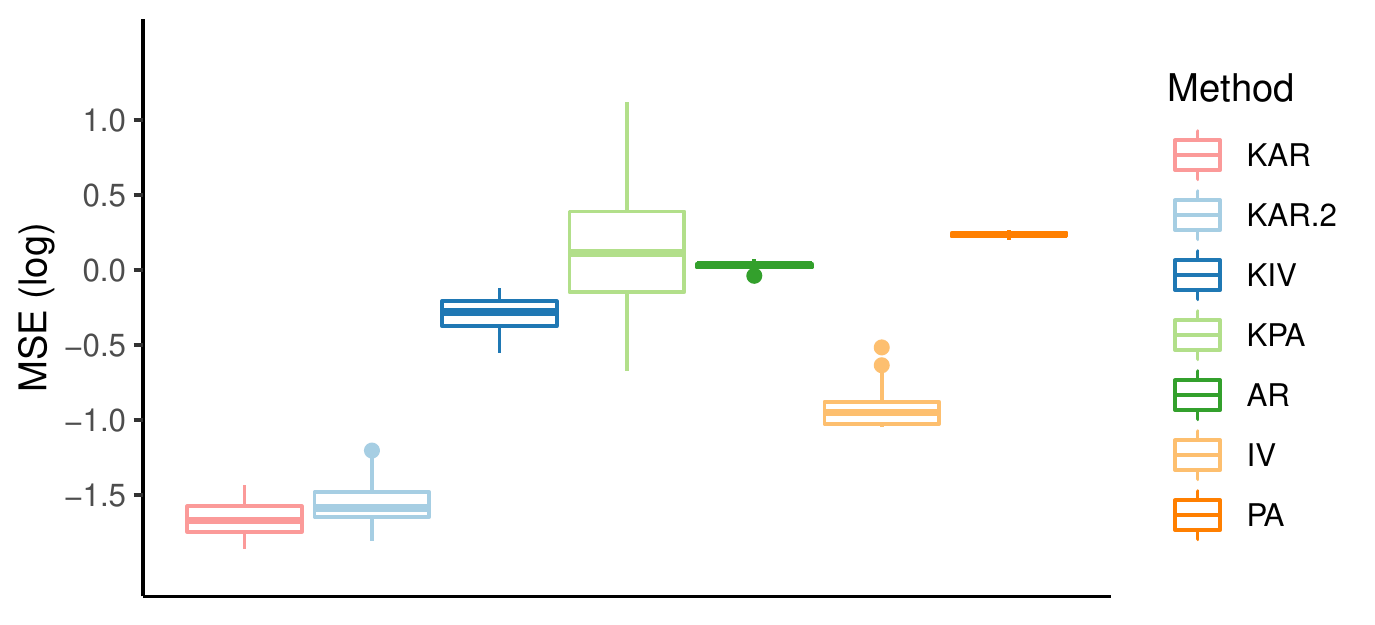}\label{fig:mse_synthetic}} 
        % \vspace{-0.3cm}
    \caption{
    % \small 
    % Fitted models:
    Restuls for the synthetic example:
    fitted
    (top left) nonlinear models; and (top right) linear models;
    % KAR, KAR.2, KIV, KPA, AR, IV and PA estimators.
    (bottom): log MSE.}
    \vspace{-0.43cm}
    \label{fig:kar_fit}
\end{figure}

% We conduct simulation to evaluate kernel anchor regression. 
\paragraph{Fitting methods}
We consider estimations via the three-stage kernel anchor regression with disjoint data set projection (\textbf{KAR}) and two-stage kernel anchor regression with joint data set projection
% algorithm 
(\textbf{KAR.2}). The baseline approaches include the kernel-based nonlinear methods: kernel instrument variable regression (\textbf{KIV}), kernel partialling out regression (\textbf{KPA}), kernel ridge regression (\textbf{KReg}); and the linear models: linear anchor regression (\textbf{AR}), linear instrument variable regression (\textbf{IV}), linear partialling out regression (\textbf{PA}) and ordinary least square (\textbf{OLS}). 
We use Gaussian kernel for all kernel methods, where the median heuristic is used for choosing the bandwidth \citep{gretton2012kernel}. 
% lengthscales are set according to median interpoint distance. 
For the synthetic example, we set $\gamma = 2$ for all anchor regressions (\textbf{KAR}, \textbf{KAR.2} and \textbf{AR}). 

% The fitted model is shown in \Cref{fig:kar_fit}. From the result, we can see that the \textbf{KAR} produce the closest estimation  to the true model among all methods, and outperforms \textbf{KAR.2} in this scenario. The comparison with linear models are also shown. In this case, IV model fitted better than other linear models. 
% estimation result is shown in Figure~\ref{fig::1}.
For each algorithm, we 
% then
implement 50 trials 
% simulations 
and calculate the mean squared error (MSE) with respect to the true causal model $\Ex( Y | do(x))$\footnote{Setting a particular value $X=x$ while ignoring other variables that may potentially changing the distribution of $y$, $p(y|X=x)$ is noted as $p(y|do(x)$ \citep{pearl2009causality, peters2016causal}. $\Ex[Y|do(x)]$ is set us the mean over $p(y|do(x))$ averaging out  different $Z$ values in this case.}, which can be computed from the structural model. 
%{\wenqi Can we put $p(y | do(x))$ here?} \wk{To clarify: let $f(x)$ be the fitted model; so the MSE is by comparing $\int_X(f(x) - y)^2p(y|x)$? If I understand this correctly, I agree with the $do(x)$ notation. } {\wenqi  MSE: $\int_X(f(x) - \Ex(Y|do(x)))^2dx$. }
% (the red line in \Cref{fig:kar_fit}).
A trial is shown in \Cref{fig:kar_fit} as a visual example. We can see that the \textbf{KAR} produces a closest estimation to the true model among all other methods and outperforms \textbf{KAR.2}. The comparison with linear models are also shown. IV model fits better than other linear models.
We report $\log_{10}(\text{MSE})$ in the bottom of \Cref{fig:kar_fit}, which shows
% As shown in the left side of Figure~\ref{fig::2}, 
that both KAR methods 
% perform better 
have smaller errors than others. \textbf{KAR} performs slightly better than \textbf{KAR.2} in this case. 
To check
the robustness of KAR estimators, we study a less smooth variant of the generating process and show the results in \Cref{app:variant}.

\begin{figure}[t!]
    \centering
    % \vspace{-0.2cm}
\includegraphics[width=0.75\textwidth]{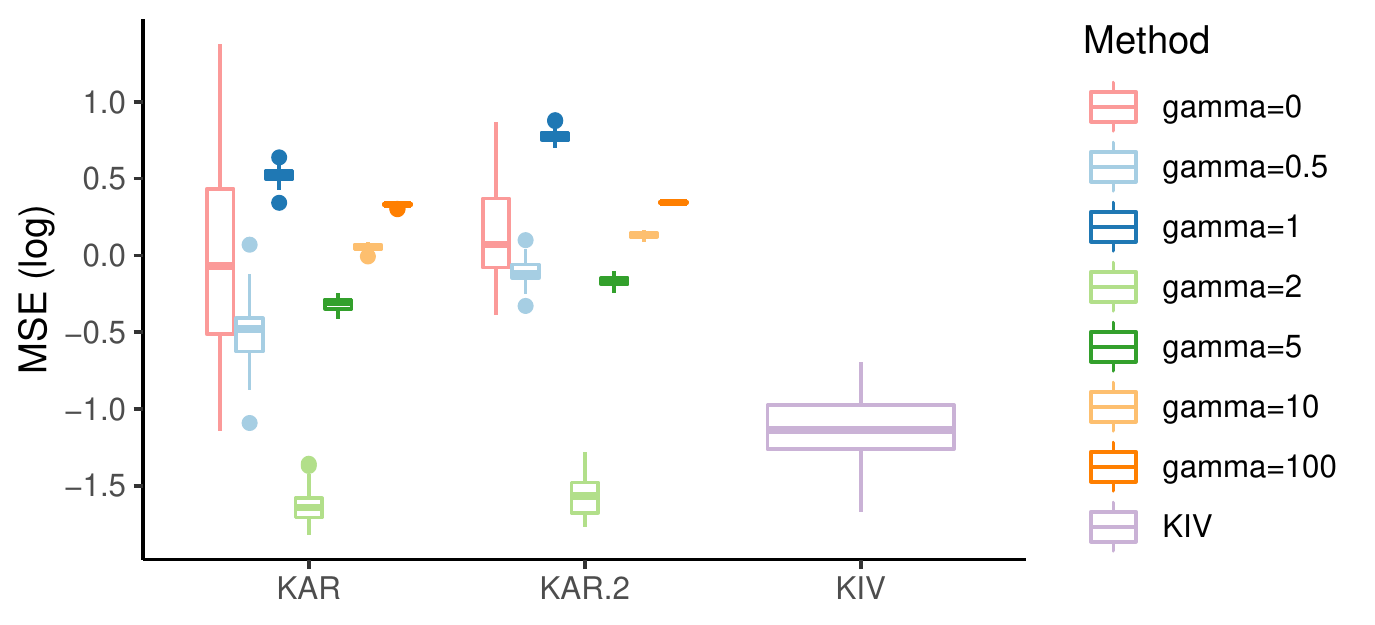}
    % \vspace{-0.3cm}
    \caption{
    % \small 
    Effects of different $\gamma$ choices on MSE.}
    \label{fig:MSE_IV}
    % \vspace{-0.5cm}
\end{figure}
\paragraph{The effect of $\gamma$ choices}
%{\wenqi Well, this figure is for KIV setting. I can include different $\gamma$s for the original case if necessary.}
%\wk{Indeed. that's me making up the story to merge the KIV setting and say KAR can be better. we don't need different $\gamma$ for original case in the main text. prob no space. if result interesting, we can do it in appendix.}
To investigate how the change of $\gamma$ affects the estimator, 
we consider KIV as our baseline as the IV setting corresponds to $\gamma \to \infty$. We consider the data generating process used in the KIV paper \citep{singh2019kernel}.
% whose details are included in \Cref{app:kiv}. 
% to show that KAR with appropriate choices of $\gamma$ performs better.
The $\log_{10}$(MSE) results of \textbf{KAR} and \textbf{KAR.2}, in comparison with \textbf{KIV}, are shown in \Cref{fig:MSE_IV}. For the simulation, we set $N=1000$, $n_1 = n_2 = 200$,  $n=n_1+n_2 =400$ and $m = 600$. From the result, we can see that both \textbf{KAR} and \textbf{KAR.2} achieves smaller error when choosing $\gamma=2$.
Data generation and model implementation details are included in \Cref{app:kiv}.

\begin{figure}[t!]
    \centering
        % \vspace{-0.3cm}
    {\includegraphics[width=0.7\textwidth]{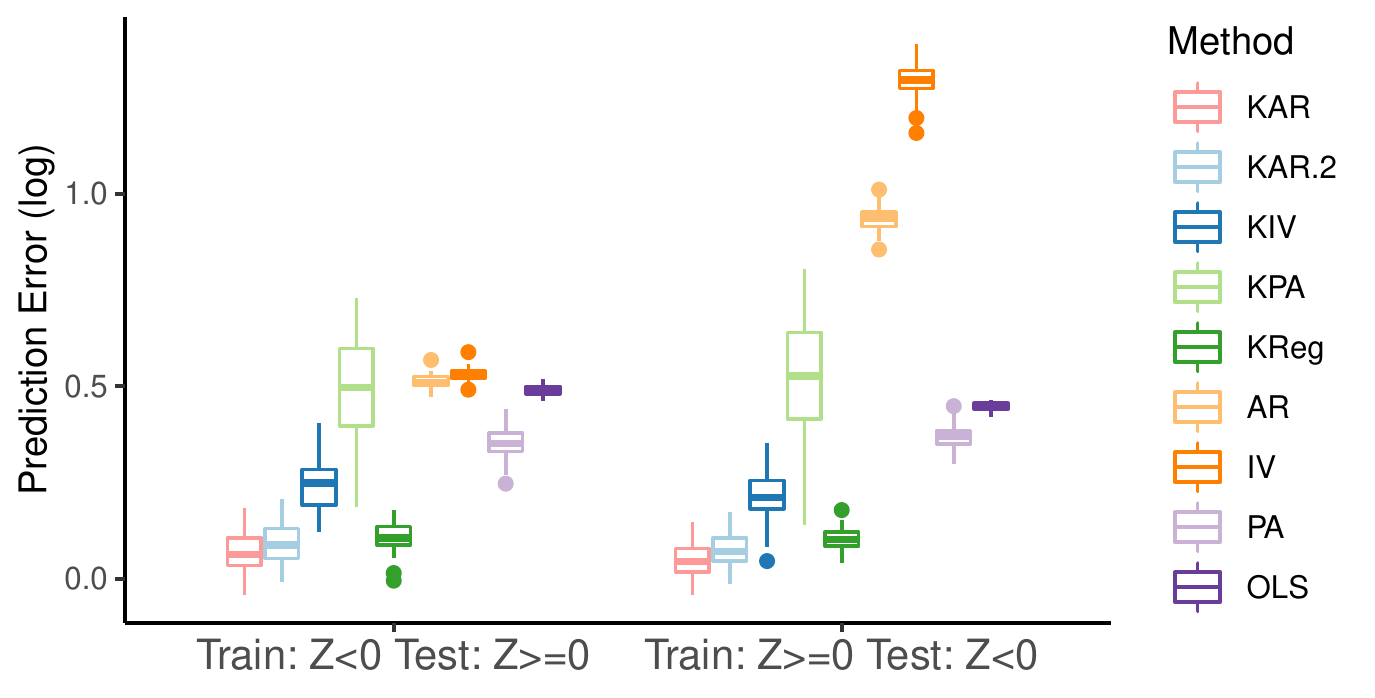}}
    % \vspace{-0.35cm}
    \caption{
    % \small 
    Prediction error with distributional intervention.}\label{fig:PE_synthetic}
    % \vspace{-0.35cm}
\end{figure}

\paragraph{Intervention and Generalization}
To evaluate the robustness and generalization performance of both KAR estimators under distribution shift, as discussed in \citet{rothenhausler2018anchor}, we intervene the anchor variable $Z$. We train the model on a subpopulation of samples with $Z<0$ and test on the samples with $Z \geq 0$. 
The performance is measured by prediction error (PE) of fitted model with respect to $\Ex(Y|X=x, Z \geq 0)$, 
%\wk{again to clarify, $f(x)$ is fitted model, PE is $\int_X\int_{z=0}^{0.5} (y - f(x))^2 p(y|x, z)dzdx$, correct?} {\wenqi PE: $\int_X (f(x) - \Ex(Y|X=x, Z \geq 0))^2 p(x|z \geq 0)dx$}, 
where the true conditional model is not known in closed form but estimated from samples. 

We also exchange the training set and the testing set.
%We also do it in the opposite way.
%\td{check is this correct}.
As shown in \Cref{fig:PE_synthetic}, our \textbf{KAR} estimator has the lowest PE among others, 
% as well as very similar PE in the two cases,
% under such distribution perturbation, 
showing better out-of-distribution generalization performance. More importantly, by checking the 
% PE performance on 
two 
% different 
(flipped) scenarios, i.e. train on $Z<0$ v.s. train on $Z\geq 0$, we also see that \textbf{KAR} is the most invariant in terms of PE. On the contrary, linear version of \textbf{AR} and \textbf{IV} achieves very different PE in both cases. 
Variances of PE for \textbf{KPA} are also very different in the two cases. Despite \textbf{KReg} achieves a relatively low PE in both cases, the distributions of PE can be found very different.

% \begin{figure}[t!]
%     \centering
%     \includegraphics[width=0.49\textwidth]{aistats23/PEc.pdf}\label{fig::2}
%     \caption{Prediction errors of all estimators for the first case (training on data with $Z<0$ and testing on data with $Z \geq 0$) (left) and the opposite case (training on data with $Z \geq 0$ and testing on data with $Z < 0$) (right).}
%     \label{fig::pe}
% \end{figure}

% \begin{figure}[t!]
%     \centering
%     \subfigure[MSE with different $\gamma$ choices. ]{\includegraphics[width=0.5\textwidth]{aistats23/IVcase.pdf}\label{fig:MSE_IV}}
%         \subfigure[Prediction error with interventions.]{\includegraphics[width=0.5\textwidth]{aistats23/Sim_PE_A.pdf}\label{fig:PE_synthetic}}
%     \vspace{-0.3cm}
%     \caption{Results of all estimators in original case (left) and perturbed case (right).}
%     \label{fig:res}
% \end{figure}

\begin{figure}[t!]
    \centering
    % \vspace{-0.3cm}
    \includegraphics[width=0.62\textwidth
    % , height=0.2\textwidth
    ]{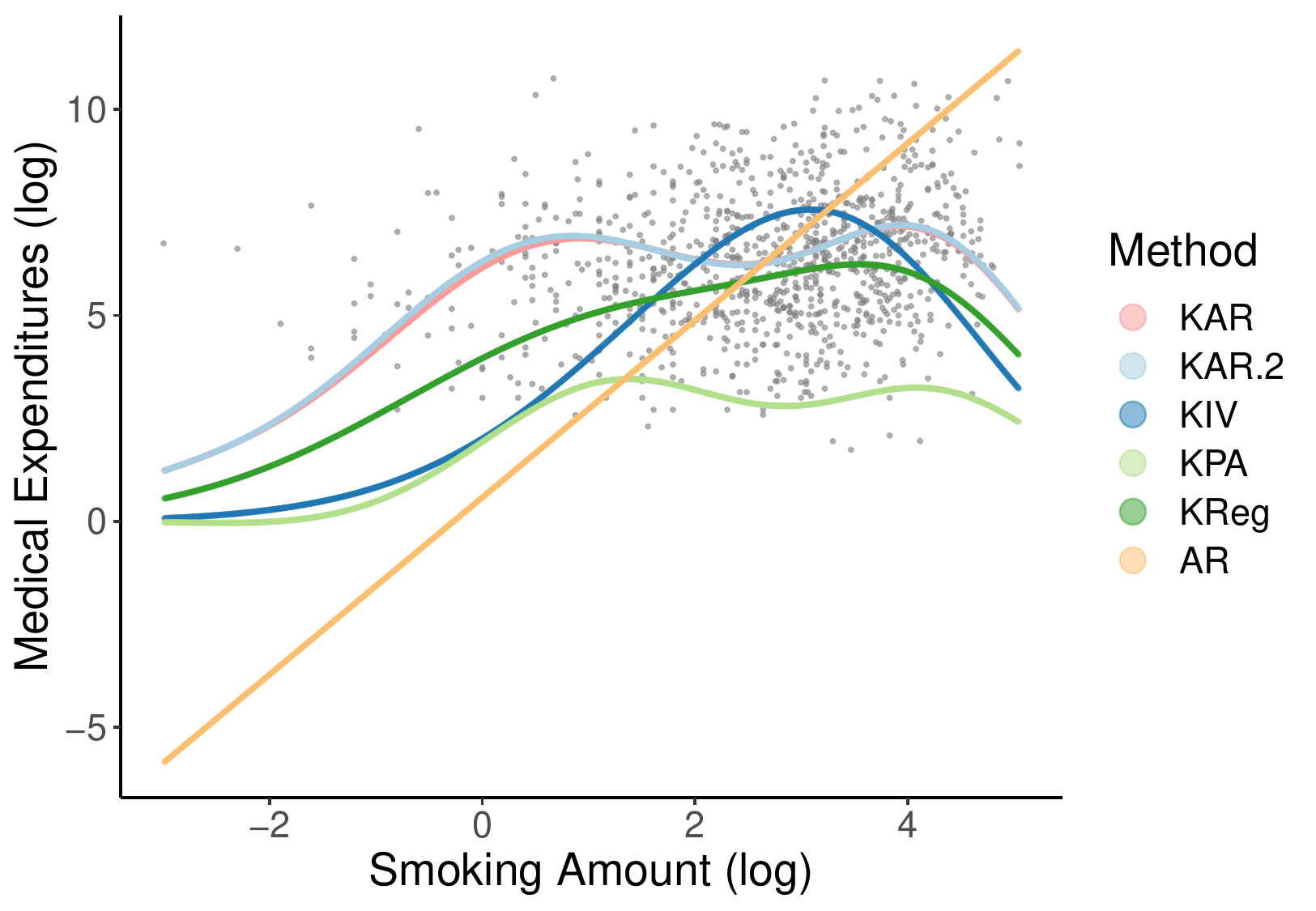}
    % \vspace{-5.1cm}
    \includegraphics[width=0.58\textwidth
    % , height=0.2\textwidth
    ]{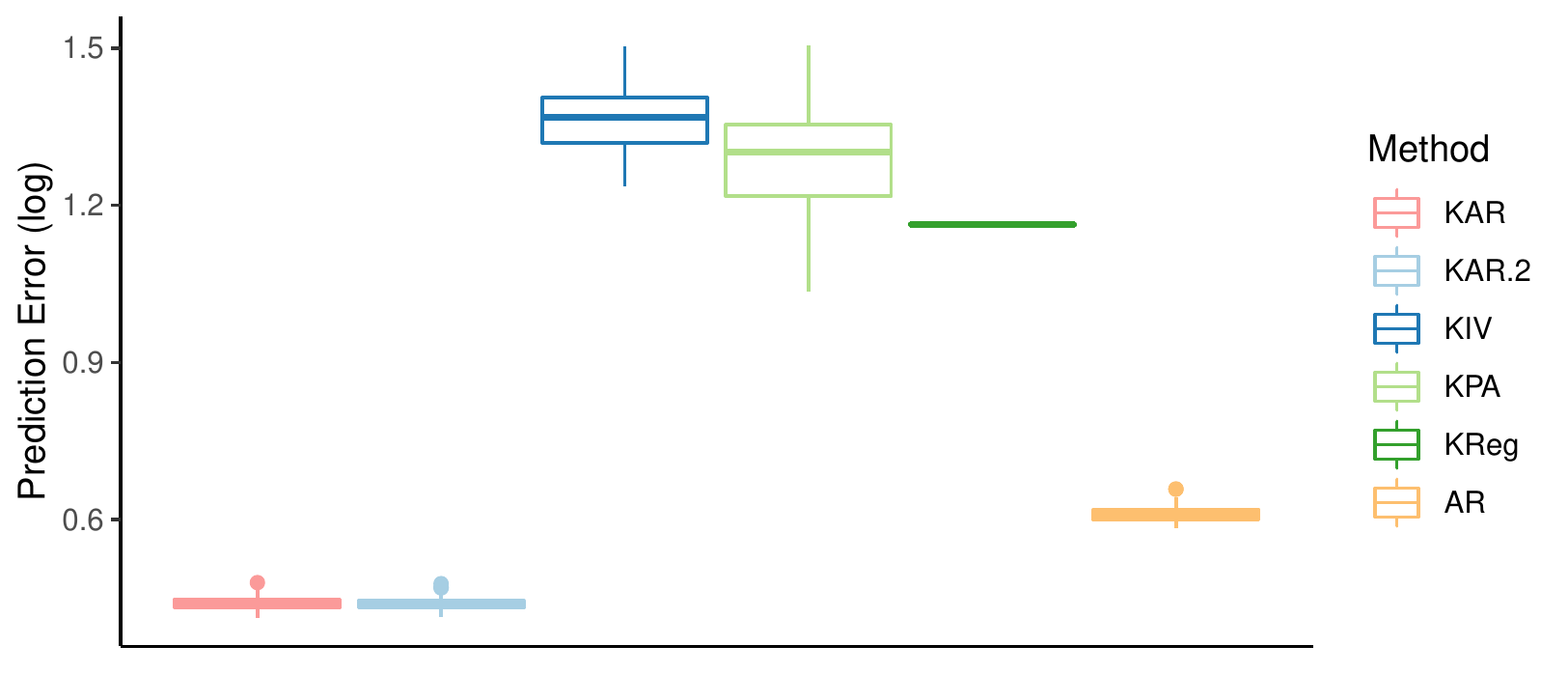}
    % \vspace{-5.23cm}
    \caption{
    % \small
    Fitted models (top)
    % of all estimators (first row) 
    and prediction errors (bottom) 
    % of all estimators 
    when training on male subjects and testing on female subjects.
    % (second row).
    }
    % \vspace{-0.35cm}
    \label{fig::sm}
\end{figure}

\subsection{Real-world application}
% For the real data application, 
We consider the smoking dataset extracted from 
% the 1987
National Medical Expenditure Survey (NMES) \citep{johnson2003disease} to study the effect of smoking amount on medical expenditure \citep{imai2004causal}\footnote{The dataset is accessible through using the \texttt{R} package for ``estimating causal dose response function'' \texttt{causaldrf} \citep{galagate2016causal} \url{https://cran.r-project.org/web/packages/causaldrf/index.html}.}.

The treatment variable $X$ is the $\log$ of smoking amount, the outcome $Y$ is the $\log$ of medical expenditure, and the anchor $Z$ is set to be the last age for smoking. 
% We conduct all the estimators applied in synthetic experiments. 
We use $1000$ samples, randomly selected from $9708$ available samples, to fit the model. We set 
% sample size
$n_1=n_2=300$, 
% for \textbf{KAR},
$n=600$ 
% for \textbf{KAR.2} 
and $m=400$. 
We also set 
$\gamma = 2.9$ and apply Gaussian kernel with median heuristic bandwidth \citep{gretton2012kernel} for all kernel methods. 
As shown in the upper part of Figure~\ref{fig::sm}, KAR estimators show that the effect of $X$ on $Y$ is more significant when $X\in [-2,1]$ compared to $X\in [1,4]$.
Our method can also be used in complement with the approaches finding causal directions, e.g.  \citep{peters2016causal}\footnote{Implementation with \texttt{R} package \texttt{CAM} can be found at \url{https://rdrr.io/cran/CAM/man/CAM.html}}. We run the CAM to ensure that there is a causal effect in the direction from $X$ to $Y$ and KAR procedure further learns the specific function representing such effect.
%from the smoking amount ($X$) to the medical expenditure ($Y$). 
% We compare our results with existing studies and find 
% The KAR procedure can 
However, existing work such as propensity score approaches \citep{imai2004causal} did not manage to extract such causal relationship between smoking 
% amount 
and medical expenditure.
% and did not find a causal relationship between these two variables.
% {\wenqi I'm not sure whether it's safe to use "improve" here.}
% the Causal Additive Model (CAM)  

To strengthen our finding, we quantify the performance of the estimators. Since we do not know the real generating process
% causal model 
% behind 
of the data, we cannot 
% compute and 
compare the MSE as \Cref{fig:kar_fit} and \ref{fig:MSE_IV}.
% results, but
% However, 
Instead, it's feasible to evaluate the performance of estimators under distribution perturbation via PE, similar to \Cref{fig:PE_synthetic}. We train models on male subjects and compute the prediction accuracy of fitted model on female subjects. The results are shown in the bottom of Figure~\ref{fig::sm}. From the result, we see that both KAR approaches outperform other kernel-based approaches as well as the linear version of AR, suggesting a better learned effect 
% for 
from the smoking 
amount 
to medical expenditure.

% See appendix for specific test settings.

\section{ Conclusion}\label{sec:conclusion}
In this work, we consider learning a more general class of causal DAG in a nonlinear setting
% features 
using 
kernelized anchor regression.
% , which is able to learn the non-linear effect, e.g. piecewise effect.
% the KAR approach. 
By considering different data splitting strategies to estimate the projection operators, 
we show that the three-stage approach not only performs better empirically than baseline approaches as well as the 2SLS approach, but also achieves optimal rate under given conditions. Identifiability results are provided and are shown to generalize KIV and ``no latent confounder'' scenarios.
% SEM with nonlinear features.
% study both the three-stage and two-stage approach for estimating KAR. We analyze the convergence properties of KAR estimators and the identifiability condition for the causal DAG. We show improved performance for our KAR estimator and demonstrate its usefulness in the real application.
% Our study also opens some future directions to better understand the anchor regression framework. While the three-stage procedure achieves optimal rate via three regression process, 
For the future, data adaptive choice of $\gamma$ can be an interesting direction to explore. 
% Moreover, the overall optimal rate for the two-stage procedure remains unclear and is interesting to further investigate.  
% \wk{The conclusion is too long, I need to shorten this.}

\subsection*{Acknowledgements}
The authors thank Ana Korba for helpful discussions. 
% W.S. ackknoledges the support by. xxx.
W.X. acknowledges the support from EPSRC grant EP/T018445/1.
% To preserve the anonymity, please include acknowledgments \emph{only} in the camera-ready papers.

% \clearpage
% \subsubsection*{References}
\bibliographystyle{apalike}
\bibliography{main}

%\iffalse
\clearpage
\appendix
\onecolumn

\renewcommand\thelemma{A\arabic{lemma}}
\setcounter{lemma}{0}
\renewcommand\theremark{A\arabic{remark}}
\setcounter{remark}{0}

\begin{center}
\LARGE
    {Supplementary Material for Nonlinear Causal Discovery \\
via Kernel Anchor Regression}
\end{center}

\section{Proofs and derivations}
%\wk{I copy paste everything here from draft.tex first}
\subsection{Proof of Theorem~\ref{thm::s3c}}
Before proving Theorem~\ref{thm::s3c}, we introduce the exact bounds of the approximation errors for estimating $E_X^p$ and $E_Y^p$ in the disjoint sample sets projection stage. Lemma~\ref{lem::as1} and ~\ref{lem:as2} below are adapted from Theorem 2 in \citet{singh2019kernel}.

\begin{lemma}\label{lem::as1}
% $\forall \alpha_1 > 0$, the solution $E_{\alpha_1,X}^{n_1}$ of the regularized empirical objective $\E_{\alpha_1, X}^{n_1}$ exists, is unique, and
% \begin{eqnarray*}
%     E_{\alpha_1,X}^{n_1} &=& (\mathbf{T}_1 + \alpha_1)^{-1} \circ \mathbf{g}_1, \\
%     \mathbf{T}_1 &=& \frac{1}{n_1} \sumni \phi(z_{1,i}) \otimes \phi(z_{1,i}), \\
%     \mathbf{g}_1 &=& \frac{1}{n_1} \sumni \phi(z_{1,i}) \otimes \psi(x_{1,i}).
% \end{eqnarray*}
Under Condition~\ref{cond::s1}, $\forall \delta \in (0,1)$, the following holds w.p. $1 - \delta$:
\begin{eqnarray*}
    \Vert E_{\alpha_1,X}^{n_1} - E_{X}^p \Vert_{\H_\Gamma} \leq
    r_{E_1}(\delta,n_1,c_1) := 
    \frac{ \sqrt{\zeta_1} (c_1 + 1)}{4^{\frac{1}{c_1+1}}} \left( \frac{4\kappa (Q_1 + \kappa \Vert E_{X}^p \Vert_{\H_{\Gamma}} \ln(2/\delta) }{\sqrt{n_1 \zeta_1}(c_1 -1)} \right)^{\frac{c_1-1}{c_1+1}},\\
    \alpha_1 = \left( \frac{8\kappa (Q_1 + \kappa \Vert E_{X}^p \Vert_{\H_\Gamma} \ln(2/\delta) }{\sqrt{n_1 \zeta_1}(c_1 -1)} \right)^{\frac{2}{c_1+1}}.
\end{eqnarray*}
\end{lemma}

\begin{lemma}\label{lem:as2}
% $\forall \alpha_2 > 0$, the solution $E_{\alpha_2,Y}^{n_2}$ of the regularized empirical objective $\E_{\alpha_2}^{n_2}$ exists, is unique, and
% \begin{eqnarray*}
%     E_{\alpha_2,Y}^{n_2} &=& (\mathbf{T}_2 + \alpha_2)^{-1} \circ \mathbf{g}_2, \\
%     \mathbf{T}_2 &=& \frac{1}{n_2} \sumni \phi(z_{2,i}) \otimes \phi(z_{2,i}), \\
%     \mathbf{g}_2 &=& \frac{1}{n_2} \sumni \phi(z_{2,i}) y_{2,i}.
% \end{eqnarray*}
Under Condition~\ref{cond::s1} and Condition~\ref{cond::s2}, $\forall \epsilon \in (0,1)$, the following holds w.p. $1 - \epsilon$:
\begin{eqnarray*}
    \Vert E_{\alpha_2,Y}^{n_2} - E_{Y}^p \Vert_{\H_\Theta} \leq
    r_{E_2}(\epsilon,n_2,c_2) := 
    \frac{ \sqrt{\zeta_2} (c_2 + 1)}{4^{\frac{1}{c_2+1}}} \left( \frac{4\kappa (Q_2 + \kappa \Vert E_{Y}^p \Vert_{\H_{\Theta}} \ln(2/\epsilon) }{\sqrt{n_2 \zeta_2}(c_2 -1)} \right)^{\frac{c_2-1}{c_2+1}},\\
    \alpha_2 = \left( \frac{8\kappa (Q_2 + \kappa \Vert E_{Y}^p \Vert_{\H_\Theta} \ln(2/\epsilon) }{\sqrt{n_2 \zeta_2}(c_2 -1)} \right)^{\frac{2}{c_2+1}}.
\end{eqnarray*}
\end{lemma}

Recall that we define the population-level risk for the regression stage $\E^\gamma(H)$, population-level risk with regularization $\E^\gamma_\xi(H)$, and the empirical risk $\widehat \E^{\gamma, m}_{\xi}(H)$ with $E_X^p$ and $E_Y^p$ being replaced by $E_{\alpha_1,X}^{n_1}$ and $E_{\alpha_2,Y}^{n_2}$, respectively. 
Denote the optimal operator to $\E^\gamma_\xi(H)$ as $H_\xi^{\gamma} = \argmin_{H} \E_\xi^{\gamma}(H)$.
We now define the empirical risk $\E_{\gamma, m}^{\xi}(H)$ with true $E_X^p$ and $E_Y^p$, and the corresponding optimal operator.
\begin{eqnarray*}
    \E_\xi^{\gamma,m}(H) = \frac{1}{m} \summ  \Vert  y_{\gamma,l} - H \psi_{\gamma,l} \Vert_{\Y}^2 + \xi \Vert H \Vert_{\H_\Omega}^2, \quad
    H_\xi^{\gamma,m} = \argmin_{H} \E_\xi^{\gamma,m}(H),
\end{eqnarray*}
where the true transformed inputs and outputs are given by $$\psi_{\gamma,l} = \psi(x_{l}) + (\sqrt{\gamma}-1)  E_{X}^p  \phi(z_{l}) \in \H_\X, \quad y_{\gamma,l} = y_{l} + (\sqrt{\gamma}-1)  E_{Y}^p \phi(z_{l}) \in \Y.$$ 

The closed form solution of $H_\xi^{\gamma,m}$ is given by Lemma~\ref{lem:as3} below, and it's adapted from Theorem 3 in \citet{singh2019kernel}

\begin{lemma}\label{lem:as3}
$\forall \xi > 0$, the solution $H_{\xi}^{\gamma,m}$ to $\E_\xi^{\gamma,m}$ exists, is unique, and
\begin{eqnarray*}
    &\mathbf{T}= \frac{1}{m} \summ T_{\psi_{\gamma,l}}, \quad
    \mathbf{g} = \frac{1}{m} \summ \Omega_{\psi_{\gamma,l}} y_{\gamma,l},
    &H_\xi^{\gamma,m} = (\mathbf{T} + \xi)^{-1} \circ \mathbf{g}.
\end{eqnarray*}
\end{lemma}

%Before showing the convergence rate of $\widehat H_\xi^{\gamma,m}$, 
We then define the following terms.
%can be used to bound the excess error.
% \begin{definition} \label{def::ABN}
% %(Definition 7 of \citep{singh2019kernel})
% The residual $\A(\xi)$, reconstruction error $\B(\xi)$, and effective dimension $\N(\xi)$ are 
% \begin{eqnarray*}
%     \A(\xi) &=& \Vert \sqrt{T} ( H_\xi^\gamma - H^\gamma ) \Vert_{\H_\Omega}^2,\\
%     \B(\xi) &=& \Vert H_\xi^\gamma - H^\gamma \Vert_{\H_\Omega}^2,\\
%     \N(\xi) &=& Tr[(T+\xi)^{-1} \circ T].
% \end{eqnarray*}
% \end{definition}

\begin{definition}\label{def::const}
%(Definition 8 in \citet{singh2019kernel})
Fix $\eta \in (0,1)$ and define the following constants
\begin{eqnarray*}
    C_{\eta} = 96 \ln^2(6/\eta), \quad
    M = 2(C + \Vert H^\gamma \Vert_{\H_\Omega} \sqrt{B}), \quad
    \Sigma = \frac{M}{2}.
\end{eqnarray*}
\end{definition}

For the excess error of \textbf{KAR} estimator $\widehat H_\xi^{\gamma,m}$, 
%in the final regression stage, 
we can bound it by five terms according to Proposition 32 in \citet{singh2019kernel}. 
%we can bound the excess error of kernel anchor regression estimator by five five terms as shown below. 
\begin{lemma}\label{lem::5term}
%(Proposition 32 in \citet{singh2019kernel})
The excess error can be bounded as follows
\begin{eqnarray*}
    \E^\gamma(\widehat H^{\gamma,m}_\xi) - \E^\gamma(H^\gamma) \leq 5 [ S_{-1} + S_0 + \A(\xi) + S_1 + S_2],
\end{eqnarray*}
where
\begin{eqnarray*}
    S_{-1} &=& \Vert \sqrt{T} \circ (\widehat{\mathbf{T}} + \xi)^{-1} (\widehat{\mathbf{g}} - \mathbf{g}) \Vert_{\H_\Omega}^2,\\
    S_0 &=& \Vert \sqrt{T} \circ (\widehat{\mathbf{T}} + \xi)^{-1} (\mathbf{T} - \widehat{\mathbf{T}}) H^{\gamma,m}_\xi \Vert_{\H_\Omega}^2,\\
    S_1 &=& \Vert \sqrt{T} \circ (\widehat{\mathbf{T}} + \xi)^{-1} (\mathbf{g} - \mathbf{T}H^\gamma) \Vert_{\H_\Omega}^2,\\
    S_2 &=& \Vert \sqrt{T} \circ (\widehat{\mathbf{T}} + \xi)^{-1} (T - \mathbf{T}) (H^\gamma_\xi - H^\gamma) \Vert_{\H_\Omega}^2,\\
    \A(\xi) &=& \Vert \sqrt{T} (H_\xi^\gamma - H^\gamma) \Vert_{\H_\Omega}^2.
\end{eqnarray*}
\end{lemma}

For all five terms above, only $\widehat{\mathbf{g}} - \mathbf{g}$ in $S_{-1}$ depends on the approximation error of $E_Y^p$. The bounds for other four terms are same to the \textbf{KIV} case. Below we introduce without proof the bond of $S_0$, $S_1$, $S_2$ and $\A(\xi)$ according to Theorem 7 in \citet{singh2019kernel}.

\begin{lemma}
\label{lem::thm7}
%(Theorem 7 in \citet{singh2019kernel})
Under Condition~\ref{cond::s1}--\ref{cond::s3}, if $m$ is large enough and $\xi \leq \Vert T\Vert_{L(\H_\Omega)}$ then $\forall \delta, \eta \in (0,1)$, the following holds up w.p. $1-\eta-\delta$:
\begin{eqnarray*}
    S_0 &\leq& \frac{4}{\xi} 4BL^2 r_x^{2\iota} \Vert H^{\gamma, m}_\xi \Vert^2_{\H_\Omega}, \\
    S_1 &\leq& 32\ln^2(6\eta) \left[ \frac{BM^2}{m^2\xi} + \frac{\Sigma^2}{m} \beta^{1/b_\gamma} \frac{\pi/b_\gamma}{\sin(\pi\b) \xi^{-1/b_\gamma} } \right],\\
    S_2 &\leq& 8 \ln^2(6/\eta) \left[ \frac{4B^2\zeta\xi^{c_\gamma-1}}{m^2\xi} + \frac{B\zeta\xi^{c_\gamma}}{m \xi} \right],\\
    \A(\xi) &\leq& \zeta \xi^{c_\gamma}.
    % \Vert H^{\gamma, m}_\xi \Vert_{\H_\Omega}^2 &\leq& \frac{16}{\xi} 6 \ln^2(6/\eta) \left[ \frac{M^2B}{m^2\xi} + \frac{\Sigma^2}{m} \beta^{1/b} \frac{\pi/b}{\sin(\pi\b) \xi^{-1/b} } \right]  \\
    % && + \frac{4}{\xi^2} 6\ln^2(6\eta) \left[ \frac{4B^2 \zeta \xi^{c-1}}{m^2} \right] + 6\zeta \xi^{c-1} + 6\Vert H^\gamma \Vert^2_{\H_\Omega}.
\end{eqnarray*}
\end{lemma}

To extend the convergence rate of \textbf{KIV} estimator to \textbf{KAR} estimator. We then illustrate the bound for $S_{-1}$.
To begin with, the bound of term $\sqrt{T} \circ ( \widehat{\mathbf{T}} + \xi)^{-1}$ in $S_{-1}$ is given by Proposition 39 in \citet{singh2019kernel}.
\begin{lemma}
\label{lem::prop39}
%(Proposition 39 in \citet{singh2019kernel})
If $\Vert \widehat \psi_{\gamma} - \psi_{\gamma} \Vert_{\H_\X} \leq r_x$ w.p. $1-\delta$, 
%$m \geq \max\left\{ \frac{2C_\eta B \N(\xi)}{\xi}, \overline{m}(\delta, c_1) \right\}$, 
$\xi \leq \Vert T \Vert_{\L(\H_\Omega)}$, $m$ is sufficiently large and Condition~\ref{cond::s3} holds, then w.p. $1-\eta/3-\delta$
$$
\Vert \sqrt{T} \circ ( \widehat{\mathbf{T}} + \xi)^{-1} \Vert_{\L(\H_\Omega)} \leq \frac{2}{\sqrt{\xi}}.
$$
\end{lemma}

With the the error propagated from the estimators in the projection stage, we can bound $\widehat \psi_\gamma - \psi_\gamma$ and $\widehat y_\gamma - y_\gamma$ as shown in Lemma~\ref{lem::xadiff}--\ref{lem::yadiff}.
%Here we give bound on $\hat{\mathbf{g}} - \mathbf{g}$, $\widehat \psi_\gamma - \psi_\gamma$ and $\widehat y_\gamma - y_\gamma$, which are different from \citep{singh2019kernel}.
\begin{lemma}\label{lem::xadiff}
Under Condition~\ref{cond::s1},
$\forall \delta \in (0,1)$, the following statement holds w.p. $1-\delta$: $\forall z \in \Z, x \in \X$,
$$
    \Vert \widehat \psi_\gamma - \psi_\gamma \Vert_{\H_\X} \leq r_x(\gamma, \delta, n_1, c_1)
    := \vert \sqrt{\gamma} - 1 \vert \kappa r_{E_1}(\delta, n_1, c_1).
$$
\end{lemma}
\begin{proof}
By definition, we have
\begin{eqnarray*}
\Vert \widehat \psi_\gamma - \psi_\gamma \Vert_{\H_\X} &=&
\Vert \left( \sqrt{\gamma} - 1 \right) \left( E^{n_1}_{\alpha_1,X} - E_{X}^p  \right) \phi(z) \Vert_{\H_\X}\\
&\leq& \vert \sqrt{\gamma} - 1 \vert \Vert E^{n_1}_{\alpha_1,X} - E_{X}^p \Vert_{\H_\Gamma} \Vert \phi(z) \Vert_{\H_\Z}.
\end{eqnarray*}
This, together with Lemma~\ref{lem::as1} and Condition~\ref{cond::s1}, ensures that w.p. $1-\delta$
$$
    \Vert \widehat \psi_\gamma - \psi_\gamma \Vert_{\H_\X} \leq r_x(\gamma, \delta, n_1, c_1)
    := \vert \sqrt{\gamma} - 1 \vert \kappa r_{E_1}(\delta, n_1, c_1).
$$
\end{proof}
\begin{remark}
Corollary 1 in \citet{singh2019kernel} is a special case of Lemma~\ref{lem::xadiff} with $\gamma=0$.
\end{remark}

\begin{lemma}\label{lem::yadiff}
Under Condition~\ref{cond::s1}--~\ref{cond::s2},
$\forall \epsilon \in (0,1)$, the following statement holds w.p. $1-\epsilon$: $\forall z \in \Z, y \in \Y$,
$$
    \Vert \widehat y_\gamma - y_\gamma \Vert_{\H_\Y} \leq r_y(\gamma, \epsilon, n_2, c_2)
    := \vert \sqrt{\gamma} - 1 \vert \kappa r_{E_2}(\epsilon, n_2, c_2).
$$
\end{lemma}
\begin{proof}
Lemma~\ref{lem::yadiff} is analogous to Lemma~\ref{lem::xadiff} by replacing $\psi_\gamma$ with $y_\gamma$. The proof is thus omitted.
\end{proof}

Combining Lemma~\ref{lem::prop39}-~\ref{lem::yadiff}, we can derive the bound of $\widehat{\mathbf{g}} - \mathbf{g}$ and then the bound of $S_{-1}$.
\begin{lemma}\label{lem::gdiff}
If $\Vert \widehat \psi_\gamma - \psi_\gamma \Vert_{\H_\X} \leq r_x$ w.p. $1 - \delta$ and $\Vert \widehat y_\gamma - y_\gamma \Vert_{\Y} \leq r_y$ w.p. $1 - \epsilon$, then w.p. $1 - \delta - \epsilon$
$$
    \Vert\widehat{\mathbf{g}} - \mathbf{g} \Vert_{\H_\Omega}^2
    \leq 3( L^2 r_x^{2\iota} r_y^2 + B^2 r_y^2 + L^2 r_x^{2\iota} C^2 ).
$$
\end{lemma}
\begin{proof}
By definition, we have
\begin{eqnarray*}
    \widehat{\mathbf{g}} - \mathbf{g} 
    &=& \frac{1}{m} \summ  \left ( \Omega_{\widehat \psi_{\gamma, l}} \widehat y_{\gamma,l} - \Omega_{ \psi_{\gamma, l}(x)}  y_{\gamma,l} \right )\\
    &=& \frac{1}{m} \summ \left\{ \Omega_{\widehat \psi_{\gamma, l}} -  \Omega_{ \psi_{\gamma, l}} \right\} \left\{ \widehat y_{\gamma,l} - y_{\gamma,l} \right\} + \Omega_{\widehat \psi_{\gamma, l}} \left\{ \widehat y_{\gamma,l} - y_{\gamma,l} \right\} + \left\{ \Omega_{\widehat \psi_{\gamma, l}} -  \Omega_{ \psi_{\gamma, l}} \right\} y_{\gamma,l}.\\
    % &=& \frac{1}{m} \summ  \left \{ \psi(x_{3,i})  + \left( \sqrt{\gamma}-1 \right) (E_{\alpha_2, X}^{n_2})^*  \phi(z_{3,i}) \right\} \left\{ y_{\a,3,i} + \left( \sqrt{\gamma}-1 \right) (E_{\alpha_1, Y}^{n_1})^* \phi(z_{3,i}) \right\}\\
    % && - \left\{ \psi(x_{3,i}) + \left( \sqrt{\gamma}-1 \right) (E_{\rho, X})^* \phi(z_{3,i}) \right\} \left\{ y_{\a,3,i} + \left( \sqrt{\gamma}-1 \right) (E_{\rho, Y})^* \phi(z_{3,i})\right\}\\
\end{eqnarray*}
We then have
\begin{eqnarray*}
    \Vert \widehat{\mathbf{g}} - \mathbf{g} \Vert_{\H_\Omega}^2 
    &\leq& \frac{3m}{m^2} \summ \Vert \left\{ \Omega_{\widehat \psi_{\gamma, l}} -  \Omega_{ \psi_{\gamma, l}} \right\} \left\{ \widehat y_{\gamma,l} - y_{\gamma,l} \right\} \Vert_{\H_\Omega}^2 + \Vert \Omega_{\widehat \psi_{\gamma, l}} \left\{ \widehat y_{\gamma,l} - y_{\gamma,l} \right\} \Vert_{\H_\Omega}^2  \\
    && + \Vert \left\{ \Omega_{\widehat \psi_{\gamma, l}} -  \Omega_{ \psi_{\gamma, l}} \right\} y_{\gamma,l} \Vert_{\H_\Omega}^2 \\
    &\leq& \frac{3}{m} \summ \Vert  \Omega_{\widehat \psi_{\gamma, l}} -  \Omega_{ \psi_{\gamma, l}} \Vert_{\L(\Y, \H_\Omega)}^2  \Vert \widehat y_{\gamma,l} - y_{\gamma,l}  \Vert_{\Y}^2 + \Vert \Omega_{ \psi_{\gamma, l}} \Vert_{\L(\Y, \H_\Omega)}^2 \Vert \widehat y_{\gamma,l} - y_{\gamma,l}  \Vert_{\Y}^2 \\
    && + \Vert  \Omega_{\widehat \psi_{\gamma, l}} -  \Omega_{ \psi_{\gamma, l}} \Vert_{\L(\Y, \H_\Omega)}^2 \Vert y_{\gamma,l} \Vert_{\Y}^2.
\end{eqnarray*}
By the boundedness and the Hölder property in Condition~\ref{cond::s3}, we obtain that w.p. $1-\delta-\epsilon$,
\begin{eqnarray*}
    \Vert \widehat{\mathbf{g}} - \mathbf{g} \Vert_{\H_\Omega}^2 
    &\leq& \frac{3}{m} \summ L^2 \Vert \widehat \psi_{\gamma, l} -  \psi_{\gamma, l}\Vert_{\H_\X}^{2\iota} \Vert \widehat y_{\gamma,l} - y_{\gamma,l}  \Vert_{\Y}^2 + \Vert \Omega_{ \psi_{\gamma, l}} \Vert_{\L(\Y, \H_\Omega)}^2 \Vert \widehat y_{\gamma,l} - y_{\gamma,l}  \Vert_{\Y}^2 \\
    && + L^2 \Vert \widehat \psi_{\gamma, l} -  \psi_{\gamma, l}\Vert_{\H_\X}^{2\iota} \Vert y_{\gamma,l} \Vert_{\Y}^2\\
    &\leq& 3( L^2 r_x^{2\iota} r_y^2 + B^2 r_y^2 + L^2 r_x^{2\iota} C^2).
\end{eqnarray*}
\end{proof}

\begin{lemma}
\label{lem::S-1}
Under Condition~\ref{cond::s1}--\ref{cond::s3}, then w.p. $1-\delta-\epsilon$
$$
    S_{-1} \leq \frac{4}{\xi} 3( L^2 r_x^{2\iota} r_y^2 + B^2 r_y^2 + L^2 r_x^{2\iota} C^2 ).
$$
%  S_{-1} &=& \Vert \sqrt{T} \circ (\hat{\mathbf{T}} + \xi)^{-1} (\hat{\mathbf{g}} - \mathbf{g}) \Vert_{\H_\Omega}^2,\\
\end{lemma}
\begin{proof}
We can derive from the definition of $S_{-1}$ that
$$
    S_{-1} \leq \Vert \sqrt{T} \circ (\widehat{\mathbf{T}} + \xi)^{-1} \Vert_{\L(\H_\Omega)}^2  \Vert \widehat{\mathbf{g}} - \mathbf{g} \Vert_{\H_\Omega}^2.
$$
This, together with Lemma~\ref{lem::prop39} and Lemma~\ref{lem::gdiff}, ensures
$$
    S_{-1} \leq \frac{4}{\xi} 3( L^2 r_x^{2\iota} r_y^2 + B^2 r_y^2 + L^2 r_x^{2\iota} C^2 ).
$$
\end{proof}

We then show the order of the sum $S_{0} + S_1 + S_2 + \A(\xi)$, which is adapted from Theorem 4 in \citet{singh2019kernel}.
\begin{lemma}
\label{lem::4term}
%(Theorem 4 in \citet{singh2019kernel})
Under Condition~\ref{cond::s1}-- ~\ref{cond::s3}, choose $\alpha_1 = n_1^{-\frac{1}{c_1+1}}$, $n_1 = m^{\frac{d_1(c_1+1)}{\iota(c_1-1)}}$, where $d_1 > 0$. Let 
$$
    f(m) = \frac{1}{m^{2+\di}\xi^3} + \frac{1}{m^{1+\di}\xi^{2+1/b_\gamma}} + \frac{1}{m^\di \xi} + \xi^{c_\gamma} + \frac{1}{m^2\xi} + \frac{1}{m\xi^{1/b_\gamma}},
$$
we then have
$$
    \op(S_0 + \A(\xi) + S_1 + S_2) = O(f(m)).
$$
\begin{itemize}
    \item [\textup(i)] If $\di \leq \frac{b_\gamma({c_\gamma}+1)}{{b_\gamma c_\gamma}+1}$ then $O(f(m)) = O(m^{-\frac{\di {c_\gamma}}{{c_\gamma}+1}})$ with $\xi = m^{-\frac{\di}{{c_\gamma}+1}}$;
    
    \item [\textup(ii)] If $\di > \frac{b_\gamma({c_\gamma}+1)}{{b_\gamma c_\gamma}+1}$ then $O(f(m)) = O(m^{-\frac{b_\gamma {c_\gamma}}{{b_\gamma c_\gamma}+1}})$ with $\xi = m^{-\frac{b_\gamma}{{b_\gamma c_\gamma}+1}}$.
\end{itemize}
\end{lemma}

% \begin{proof}
% [Proof of Theorem~\ref{thm::s3e}]

% The choice of $\alpha_2$ and $n_2$ in the statement of Theorem~\ref{thm::s3e} ensure that
% $$
% r_y^2 = O([(n_2^{-\frac{1}{2}})^{\frac{2}{c_2+1}}]^2) = O(m^{-\dii})
% $$

% Ignoring constants in Lemma~\ref{lem::gdiff}, by Lemma~\ref{lem::4term} we have
% $$
% \E(\hat H^m_\xi) - \E(H^\gamma) = \op(f(m) + \frac{1}{\xi}())
% $$

% Let $\xi = m^{-e}$, we have $e>0$ as $\xi \rightarrow 0$. Note that $f(m) +  m^{-\dii}\xi^{-1}$ is only one term more than $f(m)$ in Lemma~\ref{lem::4term}. Therefore, we only need to include the case where the extra term $m^{-\dii}\xi^{-1}$ has the highest order, which requires
% \begin{eqnarray}
% -\dii + e &\geq& -(2+\di) + 3e,
% \label{eq::1}\\
% -\dii + e &\geq& -(1+\di) + (2+\frac{1}{b}e),
% \label{eq::2}\\
% -\dii + e &\geq& -\di + e,
% \label{eq::3}\\
% -\dii + e &\geq& -ce,
% \label{eq::4}\\
% -\dii + e &\geq& -2+e,
% \label{eq::5}\\
% -\dii + e &\geq& -1 + \frac{1}{b}e.
% \label{eq::6}
% \end{eqnarray}

% Note that $\di, \dii, e >0$, $b >1$ and $c \in (1,2]$, so $  b(1-c) <  2 \Leftrightarrow (b(c+1))/({bc+1}) < 2 $.

% For case (i)(a), $d_1, d_2 \leq \frac{b(c+1)}{bc+1}$ and $d_2 \leq \min (d_1, \frac{b(c+1)(d_1+1)}{bc+2b+1})$,
% \end{proof}

\begin{proof}
[Proof of Theorem~\ref{thm::s3c}]
The choices of $\alpha_1, \alpha_2$ and $n_1, n_2$ in the statement of Theorem~\ref{thm::s3c} ensure that
\begin{eqnarray*}
r_x^2 = O([(n_1^{-\frac{1}{2}})^{\frac{2}{c_1+1}}]^{2\iota}) = O(m^{-\di}), \quad
r_y^2 = O([(n_2^{-\frac{1}{2}})^{\frac{2}{c_2+1}}]^2) = O(m^{-\dii}).
\end{eqnarray*}
Thus, by Lemma~\ref{lem::S-1}, we have $$\op(S_{-1}) = \op(1/{\xi}( r_x^{2\iota} r_y^2 + r_y^2 + r_x^{2\iota})) = \op(1/{\xi} \left\{ m^{-\di} + m^{-\dii} + m^{-\di-\dii} \right\} ).$$ Since $\di, \dii >0$, and $\di \leq \dii$ by Condition~\ref{cond::s1&2}, $m^{-\di}/{\xi}$ then dominates two other terms in $S_{-1}$. 

Note that $f(m)$ in Lemma~\ref{lem::4term} also includes $m^{-\di}/{\xi}$. Therefore, given Condition~\ref{cond::s1&2}, the sum of four terms $S_0 + \A(\xi) + S_1 + S_2$ dominates $S_{-1}$, which suggests that the approximation error of $E_Y^p$ is dominated by that of $E_X^p$. We can then derive the result from Lemma~\ref{lem::4term}.
\end{proof}

\subsection{Proof of Theorem~\ref{thm::causal}}
\begin{proof}
[Proof of Theorem~\ref{thm::causal}]

Under the kernel structural equation model, simple calculation gives
\begin{align}
C =& B_{CZ}\Phi(Z) +\epsilon_C, \label{eq:sem_c}\\
\Psi(X) = &(B_{XZ} +B_{XC}B_{CZ}) \Phi(Z) + B_{XC}\epsilon_C + \epsilon_X, \label{eq:sem_x}\\
Y = &[ B_{YZ} + B_{YC}B_{CZ} + B_{YX}(B_{XZ} + B_{XC}B_{CZ}) ] \Phi(Z) \nonumber \\
&+ (B_{YC} +B_{YX}B_{XC})\epsilon_C + B_{YX}\epsilon_X + \epsilon_Y.\label{eq:sem_y}
\end{align}
We denote $B_{\square \triangle}$ as the adjoint operator of $B_{\triangle\square}$, $B_{\square \triangle} = B_{\triangle \square}^*$. When no ambiguity arise, we use the transpose matrix notation $B_{\square \triangle} = B_{\triangle \square}^\T$. For instance, $B_{XZ} = B_{ZX}^\T$, $B_{YC} = B_{CY}^\T$. 
Recall that the transformed input and output in \Cref{eq:transform_x} and \Cref{eq:transform_y} has the form
$$
\psi_\gamma(X) = \psi(X) - E^p_{X} \phi(Z) + \sqrt{\gamma} E^p_{X} \phi(Z),
$$
and 
$$
Y_\gamma = Y - E^p_{Y} \phi(Z) + \sqrt{\gamma} E^p_{Y} \phi(Z).
$$
In the SEM case, the projections $E_X^p$ and $E_Y^p$ into $\phi(Z)$ are noted by the (composition of) operators in \Cref{eq:sem_x} and \Cref{eq:sem_y}, where 
$$E_X^p = (B_{XZ} +B_{XC}B_{CZ}),$$ 
and 
$$
E_Y^p = [ B_{YZ} + B_{YC}B_{CZ} + B_{YX}(B_{XZ} + B_{XC}B_{CZ}) ]
.$$ As such, the transformed input and output has the form
\begin{equation}\label{eq:transform_x_res}
\psi_\gamma(x) = B_{XC}\epsilon_C + \epsilon_X + \sqrt{\gamma} (B_{XZ} +B_{XC}B_{CZ}) \phi(Z),
\end{equation}
and 
\begin{equation}\label{eq:transform_y_res}
y_\gamma = (B_{YC} +B_{YX}B_{XC})\epsilon_C + B_{YX}\epsilon_X + \epsilon_Y + \gamma [ B_{YZ} + B_{YC}B_{CZ} + B_{YX}(B_{XZ} + B_{XC}B_{CZ}) ] \phi(Z).  
\end{equation}
Define relevant covariance matrix/operators as $\Sigma_C = \Ex[\epsilon_C \epsilon_C^\T]$, $\Sigma_X = \Ex[\epsilon_X \otimes \epsilon_X]$ and $\Sigma_Z = \Ex[\phi(Z) \otimes \phi(Z)]$, where $\otimes$ denotes the tensor outer product. Then the solution for the least square objective on the transformed input output can be written as
$$H^{\gamma} = \Ex[Y_\gamma \psi_\gamma(X)] (\Ex[\psi_\gamma(X) \otimes \psi_\gamma(X)])^{-1}. $$
Plug in the transformed terms in the form of \Cref{eq:transform_x_res} and \Cref{eq:transform_y_res}, we have
\begin{align*}
    & \Ex[\psi_\gamma(X) \otimes \psi_\gamma(X)] \\
    &= \Ex[ ( B_{XC}\epsilon_C + \epsilon_X + \sqrt{\gamma} (B_{XZ} +B_{XC}B_{CZ}) \phi(Z)) ( B_{XC}\epsilon_C + \epsilon_X \\
    & \qquad + \sqrt{\gamma} (B_{XZ} +B_{XC}B_{CZ}) \phi(Z))^\T] \\
    &= B_{XC} \Ex[\epsilon_C \epsilon_C^\T] B_{CX} + \Ex[\epsilon_X \otimes \epsilon_X] \\
    & \quad + \gamma (B_{XZ} +B_{XC}B_{CZ}) \Ex[\phi(Z)\otimes \phi(Z)] (B_{ZX} +B_{ZC}B_{CX})\\
    &= B_{XC}\Sigma_C B_{CX} + \Sigma_X + \gamma (B_{XZ} +B_{XC}B_{CZ}) \Sigma_Z (B_{ZX} +B_{ZC}B_{CX}).
\end{align*}
Moreover, $\Ex[Y_\gamma  \psi_\gamma(X)]$ has the form
\begin{align*}
    &(B_{YC} +B_{YX}B_{XC})\Ex[\epsilon_C \epsilon_C^\T]B_{CX} + B_{YX}\Ex[\epsilon_X \otimes \epsilon_X] + \\
    &\gamma [ B_{YZ} + B_{YC}B_{CZ} + B_{YX}(B_{XZ} + B_{XC}B_{CZ}) ] \Ex[\phi(Z) \otimes \phi(Z)](B_{ZX} + B_{ZC}B_{CX})\\
    =& (B_{YC} +B_{YX}B_{XC})\Sigma_C B_{CX} + B_{YX}\Sigma_X + \\
    &\gamma [ B_{YZ} + B_{YC}B_{CZ} + B_{YX}(B_{XZ} + B_{XC}B_{CZ}) ] \Sigma_Z (B_{ZX} + B_{ZC}B_{CX})
\end{align*}
as $\epsilon_C$, $\epsilon_X$ and $\epsilon_Y$ are independent 
% mean zero 
variables, which are also independent of $Z$. In overall, we have
\begin{align*}
    H^{\gamma} =& [(B_{YC} +B_{YX}B_{XC})\Sigma_C B_{CX} + B_{YX}\Sigma_X \\
    &+\gamma [ B_{YZ} + B_{YC}B_{CZ} + B_{YX}(B_{XZ} + B_{XC}B_{CZ}) ] \Sigma_Z (B_{ZX} + B_{ZC}B_{CX})]\\
    &\left[ B_{XC}\Sigma_C B_{CX} + \Sigma_X + \gamma (B_{XZ} +B_{XC}B_{CZ}) \Sigma_Z (B_{ZX} +B_{ZC}B_{CX}) \right]^{-1}
\end{align*}

% By the definition of target KAR estimator $H^\gamma$ and standard regression formula we have
% \begin{eqnarray*}
% H^\gamma &=& \left[ B_{XC}\Sigma_C B_{CX}^\T + \Sigma_{X} + \gamma (B_{ZX} + B_{ZC}B_{CX}) \Sigma_Z (B_{ZX} + B_{ZC}B_{CX})^\T \right]^{-1} \\
% && [ B_{CX}\Sigma_{C}(B_{CY} + B_{CX}B_{XY})^\T + \Sigma_X B_{XY}^\T \\
% && + \gamma (B_{ZX} + B_{ZC}B_{CX}) \Sigma_Z (B_{ZY} + B_{ZC}B_{CY} + B_{ZX}B_{XY}+ B_{ZC} B_{CX}B_{XY})^\T ].
% \end{eqnarray*}

The bias of the target KAR estimator is then given by
\begin{align*}
    &H^{\gamma} - B_{YX} = \\ 
    &\Big[(B_{YC} +B_{YX}B_{XC})\Sigma_C B_{CX} + B_{YX}\Sigma_X \\
    & +\gamma [ B_{YZ} + B_{YC}B_{CZ} + B_{YX}(B_{XZ} + B_{XC}B_{CZ}) ] \Sigma_Z (B_{ZX} + B_{ZC}B_{CX})\Big]\\
    &\Big[ B_{XC}\Sigma_C B_{CX} + \Sigma_X + \gamma (B_{XZ} +B_{XC}B_{CZ}) \Sigma_Z (B_{ZX} +B_{ZC}B_{CX}) \Big]^{-1} - B_{YX}=\\
    & \Big[(B_{YC} +B_{YX}B_{XC})\Sigma_C B_{CX} + B_{YX}\Sigma_X \\
    &+\gamma [ B_{YZ} + B_{YC}B_{CZ} + B_{YX}(B_{XZ} + B_{XC}B_{CZ}) ] \Sigma_Z (B_{ZX} + B_{ZC}B_{CX})\\
    &\qquad - B_{YX}(B_{XC}\Sigma_C B_{CX} + \Sigma_X + \gamma (B_{XZ} +B_{XC}B_{CZ}) \Sigma_Z (B_{ZX} +B_{ZC}B_{CX})) \Big]\\
    &\Big[ B_{XC}\Sigma_C B_{CX} + \Sigma_X + \gamma (B_{XZ} +B_{XC}B_{CZ}) \Sigma_Z (B_{ZX} +B_{ZC}B_{CX}) \Big]^{-1}
\end{align*}
Collecting all the common terms we get 
\begin{align*}
    H^{\gamma} - B_{YX} = & \Big[\underset{\Sigma_{YX}^{\res}}{\underbrace{B_{YC}\Sigma_C B_{CX}}}
    + \gamma \underset{\Sigma_{YX}^{\para}}{\underbrace{(B_{YZ} + B_{YC}B_{CZ})\Sigma_Z(B_{ZX} + B_{ZC}B_{CX})}}\Big] 
    \\
    &\Big[ B_{XC}\Sigma_C B_{CX} + \Sigma_X + \gamma (B_{XZ} +B_{XC}B_{CZ}) \Sigma_Z (B_{ZX} +B_{ZC}B_{CX}) \Big]^{-1}
\end{align*}

% \begin{eqnarray*}
% H^\gamma - B_{XY} &=& \left[ B_{CX}\Sigma_C B_{CX}^\T + \Sigma_{X} + \gamma (B_{ZX} + B_{ZC}B_{CX}) \Sigma_Z (B_{ZX} + B_{ZC}B_{CX})^\T \right]^{-1} \\
% && \left[B_{CX}\Sigma_{C}B_{CY} + \gamma (B_{ZX} + B_{ZC}B_{CX}) \Sigma_Z (B_{ZY} + B_{ZC}B_{CY})^\T \right].
% \end{eqnarray*}

Thus, $\forall x \in \X, y\in \Y$,  consider the inner product $y^\T (H^{\gamma} - B_{YX})\psi(x) = 0$ when the following holds: (i) $B_{YC} = 0$ and $\gamma = 0$, or (ii) $B_{YZ} + B_{YC}B_{CZ} = 0$ and $\gamma = \infty$, or (iii) $B_{YC} = 0$, $B_{YZ} + B_{YC}B_{CZ} = 0$ and $\gamma \geq 0$, or (iv) $\Sigma_{YX}^\para = a \Sigma_{YX}^\res$ for some $a > 0$, and $\gamma = \infty$, or (v) $\Sigma_{XY}^\para = - a \Sigma_{XY}^\res$ for some $a > 0$, and $\gamma = 1/c$. As such, we conclude $H^\gamma = B_{XY}$.

% We the take the derivative of squared bias $\Vert H^\gamma - B_{XY} \Vert_2^2$ 
% \begin{eqnarray*}
% \frac{\partial \Vert H^\gamma - B_{XY} \Vert_2^2}{\partial \gamma} &=& -2 ( B_{CX}\Sigma_{C}B_{CY} )^\T\\
% && \left[ B_{CX}\Sigma_C B_{CX}^\T + \Sigma_{X} + \gamma (B_{ZX} + B_{ZC}B_{CX}) \Sigma_Z (B_{ZX} + B_{ZC}B_{CX})^\T \right]^{-1} \\
% && (B_{ZX} + B_{ZC}B_{CX}) \Sigma_Z (B_{ZX} + B_{ZC}B_{CX})^\T \\
% && \left[ B_{CX}\Sigma_C B_{CX}^\T + \Sigma_{X} + \gamma (B_{ZX} + B_{ZC}B_{CX}) \Sigma_Z (B_{ZX} + B_{ZC}B_{CX})^\T \right]^{-2}\\
% && \left[B_{CX}\Sigma_{C}B_{CY} + \gamma (B_{ZX} + B_{ZC}B_{CX}) \Sigma_Z (B_{ZY} + B_{ZC}B_{CY})^\T \right].
% \end{eqnarray*}

\end{proof}

\subsection{Convergence rate for KAR.2 estimator}
In this section, we will further discuss the convergence rate of \textbf{KAR.2} estimator, and show that
% we cannot derive a same or
the rate does not
improve upon the convergence rate 
% compared to 
of \textbf{KAR} estimator. 

In the three-stage KAR procedure, we approximate $E_X^p$ and $E_Y^p$ by $E_{\alpha_1, X}^{n_1}$ and $E_{\alpha_2, Y}^{n_2}$, respectively. In the two-stage KAR procedure, instead, we approximate the two operators by $E_{\alpha, X}^{n}$ and $E_{\alpha, Y}^{n}$, respectively. Note that the estimated operators $E_{\alpha, X}^{n}$ and $E_{\alpha, Y}^{n}$ use the same $\alpha$. The shared $\alpha$ may fail to ensure the optimal approximation error for $E_{\alpha, X}^{n}$ and $E_{\alpha, Y}^{n}$ at the same time.
%We note that both $E_{\alpha, X}^{n}$ and $E_{\alpha, Y}^{n}$ are based on more data than $E_{\alpha_1, X}^{n_1}$ and $E_{\alpha_2, Y}^{n_2}$ as $n = n_1 + n_2$. However, we 
\begin{lemma}\label{lem:b1}
Under Condition~\ref{cond::s1}, $\forall \delta \in (0,1)$, the following holds w.p. $1 - \delta$:
\begin{eqnarray*}
    \Vert E_{\alpha,X}^{n} - E_{X}^p \Vert_{\H_\Gamma} \leq
    r_1(\alpha):=
    \frac{4\kappa (Q_1 + \kappa \Vert E_{X}^p \Vert_{\H_\Gamma})\ln(2/\delta)}{\sqrt{n}\alpha} + \alpha^{\frac{c_1-1}{2}}\sqrt{\zeta_1}.
\end{eqnarray*}
Under Condition~\ref{cond::s1} and Condition~\ref{cond::s2}, $\forall \epsilon \in (0,1)$, the following holds w.p. $1 - \epsilon$:
\begin{eqnarray*}
    \Vert E_{\alpha,Y}^{n} - E_{Y}^p \Vert_{\H_\Theta} \leq
    r_2(\alpha):=
    \frac{4\kappa (Q_2 + \kappa \Vert E_{Y}^p \Vert_{\H_\Theta})\ln(2/\epsilon)}{\sqrt{n}\alpha} + \alpha^{\frac{c_2-1}{2}}\sqrt{\zeta_2}.
\end{eqnarray*}
Approximation error bound $r_1(\alpha)$ for $E_{\alpha,X}^{n}$ achieves its minimum at rate $O(n^{-\frac{c_1-1}{2(c_1+1)}})$ when 
$$
    \alpha = \left(\frac{8\kappa (Q_1 + \kappa \Vert E_{X}^p \Vert_{\H_\Gamma})\ln(2/\delta)}{\sqrt{n\zeta_1}(c_1-1)}  \right)^{\frac{2}{c_1+1}} = O(n^{\frac{-1}{c_1+1}});
$$
and approximation error bound $r_2(\alpha)$ for $E_{\alpha,Y}^{n}$ achieves its minimum at rate $O(n^{-\frac{c_2-1}{2(c_2+1)}})$ when 
$$
    \alpha = \left(\frac{8\kappa (Q_2 + \kappa \Vert E_{Y}^p \Vert_{\H_\Theta})\ln(2/\epsilon)}{\sqrt{n\zeta_2}(c_2-1)}  \right)^{\frac{2}{c_2+1}} = O(n^{\frac{-1}{c_2+1}}).
$$
\end{lemma}
Lemma~\ref{lem:b1} above provides the upper bounds of the approximation errors for $E_{\alpha, X}^{n}$ and $E_{\alpha, Y}^{n}$, and it's adapted from Theorem 2 in \citet{singh2019kernel}. We can see that if $c_1 \neq c_2$, we cannot claim the optimal convergence rate for $E_{\alpha, X}^{n}$ and $E_{\alpha, Y}^{n}$ at the same time, which disjoint sample sets projection estimators can guarantee by setting different $\alpha_1$ and $\alpha_2$ as shown in Lemma~\ref{lem::s1} and ~\ref{lem:s2}. In other words, in \textbf{KAR.2} procedure, the error propagated to the final stage, which are caused by using $E_{\alpha, X}^{n}$ and $E_{\alpha, Y}^{n}$, can have larger order than using $E_{\alpha_1, X}^{n_1}$ and $E_{\alpha_2, Y}^{n_2}$ separately in the \textbf{KAR} procedure. Therefore, we cannot ensure a same or improved convergence rate for \textbf{KAR.2} estimator compared to \textbf{KAR} estimator. 
% As such, the optimal rate for the two-stage procedure with joint projection is an interesting future work.

\section{Additional simulation details and results}

\renewcommand\thefigure{B\arabic{figure}}
\setcounter{figure}{0}

\subsection{Synthetic example in KIV setting}\label{app:kiv}
In this section, we show the  data generating process and implementation details for the example used in the KIV \citep{singh2019kernel} that follows the simulation case of learning counterfactual functions \citep{chen2018optimal} studied in \citet{singh2019kernel}.  The structural model is set as follows,
$$
    Y = C + \ln(|16X - 8| + 1) sgn (X - 0.5).
$$
The explanatory variables are generated from
\begin{eqnarray*}
\left( \begin{array}{c}
        C  \\
        V \\
        W \end{array} \right) &\sim&
        N \left( \left( \begin{array}{c}
        0  \\
        0 \\
        0 \end{array} \right), \left( \begin{array}{ccc}
        1, 0.5, 0  \\
        0.5, 1, 0 \\
        0, 0, 1 \end{array} \right)
        \right), \\
    X &=& F \left( \frac{W+V}{\sqrt{2}} \right),\\
    Z &=& F(W),
\end{eqnarray*}
where $F$ denote the c.d.f of standard normal distribution. This structural model ensures that anchor $Z$ is a valid instrumental variable, so that KIV is supposed to perform well in this case. We conduct kernel anchor regression with three-stage algorithm (KAR), kernel anchor regression with two-stage algorithm (KAR.2) and multiple $\gamma$s and kernel instrument variable regression (KIV). 
Set $n_1 = 200$, $n_2 = 200$, $m = 600$, $n=n_1+n_2 =400$.
For KAR and KAR.2, we set $\gamma$ to be 0, 0.5, 1, 2, 5, 10, and 100. We set $\alpha_1 = c_\alpha n_1^{-0.5}$, $\alpha_2 =  c_\alpha n_2^{-0.5}$, $\alpha =  c_\alpha n^{-0.5}$, and $\xi = 1 m^{-0.5}$, where $c_\alpha > 0$ is a constant chosen from $\{ 0.01,0.05,0.1,0.5,0.8,1,2,3\}$ for each estimator separately to minimise the corresponding MSE. We use Gaussian kernel for all kernel methods, where the lengthscales are set according to median heuristic \citep{gretton2012kernel}.
% interpoint distance.

For each algorithm, we then implement 50 simulations and calculate MSE with respect to the true causal model $\Ex (Y| do(x))$, which can be computed from the structural model. As shown in Figure~\ref{fig:kiv}, though KIV performs better than most KAR and KAR.2 estimators, KAR and KAR.2 with $\gamma = 2$ defeat KIV in the KIV setting. The parameters $c_\alpha$s are chosen to be 1, 0.1, 3, 0.8, 3, 3, 3, 1, 0.1, 3, 1, 3, 3, 3 and 2 for KAR with $\gamma$ being 0, 0.5, 1, 2, 5, 10, 100, KAR.2 with same $\gamma$ series and KIV, respectively.

\begin{figure}[t!]
    \centering
    \includegraphics[width=1.\textwidth]{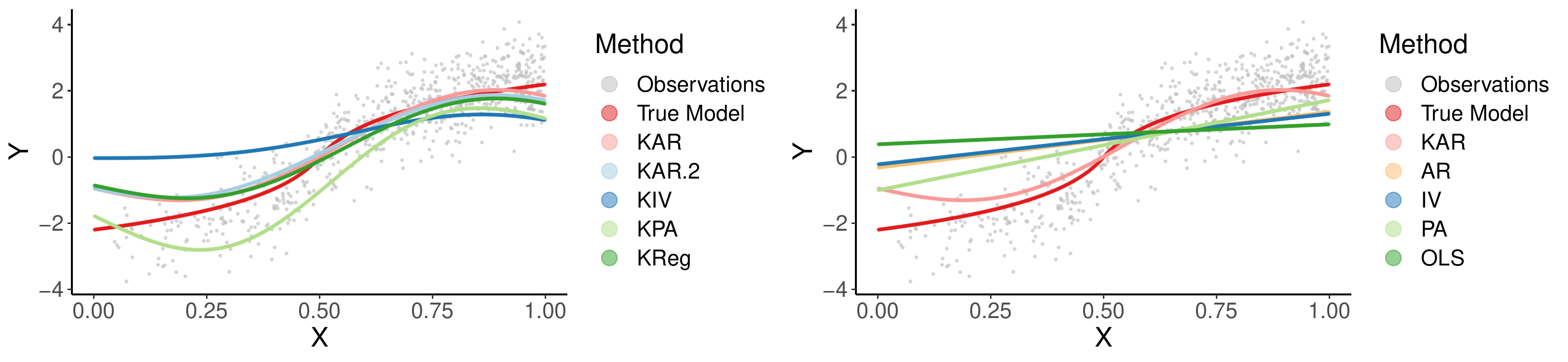}
    % \vspace{-0.3cm}
    \caption{Variant synthetic example: fitted nonlinear (left) and linear (right) methods.}
        % \subfigure[Prediction error with interventions.]{\includegraphics[width=0.5\textwidth]{aistats23/KAR2.pdf}\label{fig:PE_synthetic}}
        % \subfigure[MSE results of all estimators in the variant case. ]{\includegraphics[width=0.5\textwidth]{aistats23/MSE2-2.pdf}\label{fig::C2}}
        % \vspace{-0.3cm}
        \label{fig:variant_fitting}
\end{figure}

\subsection{A variant of the synthetic data example}\label{app:variant}
We also consider a variant case where the structural equation is same to the case in \Cref{sec:synthetic} in the main text 
\begin{equation*}
    Y = 0.75C - 0.25Z + \ln(|16X - 8| + 1) sgn (X - 0.5),
\end{equation*}
and the explanatory variables are generated as
\begin{eqnarray*}
% \left( 
\begin{pmatrix}
        C  \\
        V \\
        W \end{pmatrix} 
        % \right) 
        \sim
        N \left( 
        % \left( 
        \begin{pmatrix}
        0  \\
        0 \\
        0 \end{pmatrix}, 
        % \right), 
        % \left( 
        \begin{pmatrix}
        1, 0.3, 0.2  \\
        0.3, 1, 0 \\
        0.2, 0, 1 \end{pmatrix} 
        % \right)
        \right).\\
\end{eqnarray*}
Instead, $X$ and $Z$ are set via the following transformation.
\begin{eqnarray*}
    X = F \left( \frac{ \left|W\right|  +V}{\sqrt{2}} \right), \quad
    Z = F(\left|W\right|) - 0.5.
\end{eqnarray*}

\begin{figure}[t!]
    \centering
    % \includegraphics[width=0.48\textwidth]{aistats23/IVcase.pdf}\label{fig::2}
    % \caption{MSE results of KAR, KAR.2 and KIV estimators in KIV setting.}
        \subfigure
        [MSE results 
        % of KAR, KAR.2 and KIV estimators 
        in the KIV setting]
        {\includegraphics[width=0.7\textwidth]{aistats23/IVcase.pdf}\label{fig:kiv}}
        \subfigure
        [MSE results of all estimators in the variant case. ]
        {\includegraphics[width=0.7\textwidth]{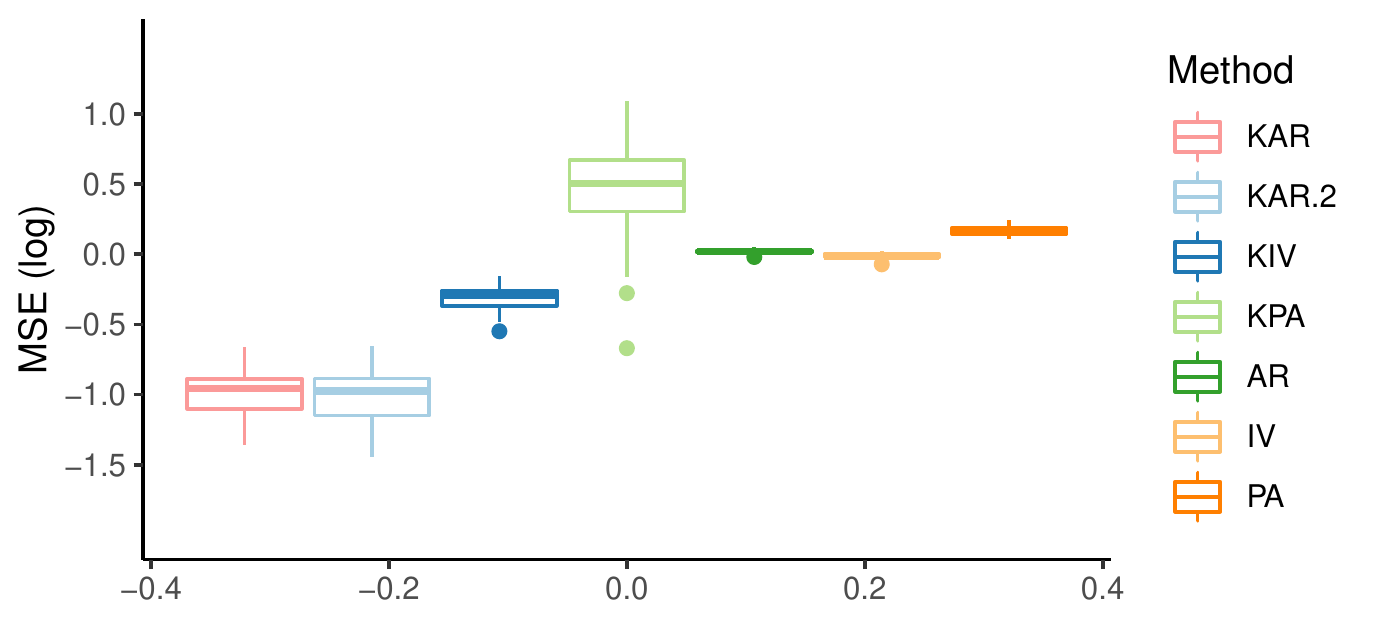}\label{fig:variant2}}
    \caption{Experimental results for additional experiments.
    % \small (left) MSE results of KAR, KAR.2 and KIV estimators in KIV setting; (right) MSE results of all estimators in the variant case. 
    }
    \label{fig::a1}
\end{figure}
The fitted result of nonlinear and linear methods is shown in Figure~\ref{fig:variant_fitting}.
The MSE averaged over 50 simulations is shown in 
Figure~\ref{fig:variant2}. From the result, we can also see that the proposed kernel anchor regression estimators still performs the best among others under the variant case.

% \begin{figure}[t!]
%     \centering
%     \includegraphics[width=0.48\textwidth]{aistats23/Sim_MSE_2.pdf}\label{fig::C2}
%     \caption{MSE results of all estimators in the variant case.}
%         % \subfigure[Prediction error with interventions.]{\includegraphics[width=0.5\textwidth]{aistats23/KAR2.pdf}\label{fig:PE_synthetic}}
%         % \subfigure[MSE results of all estimators in the variant case. ]{\includegraphics[width=0.5\textwidth]{aistats23/MSE2-2.pdf}\label{fig::C2}}
% \end{figure}
%\fi

\end{document}